\documentclass{article}
\usepackage{microtype}
\usepackage{graphicx}
\usepackage{subcaption}
\usepackage{booktabs} % for professional tables

\usepackage{hyperref}

\usepackage{yub}
\usepackage[accepted]{preprint}

\icmltitlerunning{Taylorized Training of Neural Networks}

\def\shownotes{0}  %set 1 to show author notes
\ifnum\shownotes=1
\newcommand{\authnote}[2]{$\ll$\textsf{\footnotesize #1 notes: #2}$\gg$}
\else
\newcommand{\authnote}[2]{}
\fi

\def\Unif{{\rm Unif}}

\def\cosparam{{\rm cos\_param}}
\def\cosfunc{{\rm cos\_func}}

\definecolor{C0}{HTML}{1f77b4}
\definecolor{C1}{HTML}{ff7f0e}
\definecolor{C2}{HTML}{2ca02c}
\definecolor{C3}{HTML}{d62728}
\definecolor{C4}{HTML}{9467bd}

\hypersetup{
    colorlinks,
    linkcolor={blue!50!black},
    citecolor={blue!50!black},
}
\colorlet{linkequation}{blue}

\renewcommand{\cite}{\citep}

\title{Taylorized Training: Towards Better Approximation of Neural
  Network Training at Finite Width}

\author{
  Yu Bai\thanks{Salesforce Research. {\tt yu.bai@salesforce.com}}
  \and
  Ben Krause\thanks{Salesforce Research. {\tt
      bkrause@salesforce.com}}
  \and
  Huan Wang\thanks{Salesforce Research. {\tt
      huan.wang@salesforce.com}}
  \and
  Caiming Xiong\thanks{Salesforce Research. {\tt
      cxiong@salesforce.com}}
  \and
  Richard Socher\thanks{Salesforce Research. {\tt
      rsocher@salesforce.com}}
}
\date{\today}
\begin{document}

% \maketitle
\twocolumn[
\icmltitle{Taylorized Training: Towards Better Approximation of \\
  Neural Network Training at Finite Width}

% % It is OKAY to include author information, even for blind
% % submissions: the style file will automatically remove it for you
% % unless you've provided the [accepted] option to the icml2020
% % package.

% % List of affiliations: The first argument should be a (short)
% % identifier you will use later to specify author affiliations
% % Academic affiliations should list Department, University, City, Region, Country
% % Industry affiliations should list Company, City, Region, Country

% % You can specify symbols, otherwise they are numbered in order.
% % Ideally, you should not use this facility. Affiliations will be numbered
% % in order of appearance and this is the preferred way.

\begin{icmlauthorlist}
  \icmlauthor{Yu Bai}{sf}
  \icmlauthor{Ben Krause}{sf}
  \icmlauthor{Huan Wang}{sf}
  \icmlauthor{Caiming Xiong}{sf}
  \icmlauthor{Richard Socher}{sf}  
% \icmlauthor{Bauiu C.~Yyyy}{equal,to,goo}
% \icmlauthor{Cieua Vvvvv}{goo}
% \icmlauthor{Iaesut Saoeu}{ed}
% \icmlauthor{Fiuea Rrrr}{to}
% \icmlauthor{Tateu H.~Yasehe}{ed,to,goo}
% \icmlauthor{Aaoeu Iasoh}{goo}
% \icmlauthor{Buiui Eueu}{ed}
% \icmlauthor{Aeuia Zzzz}{ed}
% \icmlauthor{Bieea C.~Yyyy}{to,goo}
% \icmlauthor{Teoau Xxxx}{ed}
% \icmlauthor{Eee Pppp}{ed}
\end{icmlauthorlist}

\icmlaffiliation{sf}{Salesforce Research, Palo Alto, CA, USA}
% % \icmlaffiliation{to}{Department of Computation, University of Torontoland, Torontoland, Canada}
% % \icmlaffiliation{goo}{Googol ShallowMind, New London, Michigan, USA}
% % \icmlaffiliation{ed}{School of Computation, University of Edenborrow, Edenborrow, United Kingdom}

\icmlcorrespondingauthor{Yu Bai}{yu.bai@salesforce.com}
\icmlcorrespondingauthor{Ben Krause}{bkrause@salesforce.com}
% % \icmlcorrespondingauthor{Eee Pppp}{ep@eden.co.uk}

% % You may provide any keywords that you
% % find helpful for describing your paper; these are used to populate
% % the "keywords" metadata in the PDF but will not be shown in the document
\icmlkeywords{Linearized Training, Neural Tangent Kernels, Deep
  Learning Theory}
\vskip 0.3in
]



\printAffiliationsAndNotice{}  % leave blank if no need to mention equal contribution

\begin{abstract}
  We propose \emph{Taylorized training} as an initiative towards better understanding neural network training at finite width. Taylorized training involves training the $k$-th order Taylor expansion of the neural network at initialization, and is a principled extension of linearized training---a recently proposed theory for understanding the success of deep learning.

  We experiment with Taylorized training on modern neural network architectures, and show that Taylorized training (1) agrees with full neural network training increasingly better as we increase $k$, and (2) can significantly close the performance gap between linearized and full training. Compared with linearized training, higher-order training works in more realistic settings such as standard parameterization and large (initial) learning rate. We complement our experiments with theoretical results showing that the approximation error of $k$-th order Taylorized models decay exponentially over $k$ in wide neural networks.
      
%     We propose \emph{Taylorized training} as an initiative towards better understanding neural network training at finite width. Taylorized training involves training the $k$-th order Taylor expansion of the neural network at initialization, and is a principled extension of linearized training---a recently proposed theory for understanding the success of deep learning.

%    We experiment with Taylorized training on modern neural net architectures, and show that Taylorized training (1) agrees with full neural net training increasingly better as we increase $k$, and (2) can significantly close the performance gap between linearized and full training. Compared with linearized training, higher-order training works in more practical settings such as standard parameterization and large (initial) learning rate, suggesting their broader applicability. We complement our experiments with theoretical results showing that the approximation error of $k$-th order Taylorized models decay exponentially over $k$ in wide neural networks.
\end{abstract}
\section{Introduction}
\label{section:intro}
Deep learning has made immense progress in solving artificial
intelligence challenges such as computer vision, natural language
processing, reinforcement learning, and so
on~\cite{lecun2015deep}. Despite this great success, fundamental
theoretical questions such as why deep networks train and
generalize well are only partially understood. 

A recent surge of research establishes the connection between
% the training trajectory of 
wide neural networks and their linearized
models. It is shown that wide neural networks can be trained in
a setting in which each individual weight only moves very slightly
(relative to itself), so that the evolution of the network can be
closely approximated by the evolution of the linearized model, which
when the width goes to infinity has a certain statistical limit
governed by its Neural Tangent Kernel (NTK).
Such a connection has led to provable optimization and generalization
results for wide neural nets
~\cite{li2018learning,jacot2018neural,du2018gradient,du2019gradient,zou2019gradient,lee2019wide,arora2019fine,allen2019learning},
and has inspired the design of new algorithms such as neural-based
kernel machines that achieve competitive results on benchmark learning
tasks~\cite{arora2019exact,li2019enhanced}.

% While the theory of linearized training is powerful, it is
While linearized training is powerful in theory, it is questionable
whether it really explains neural network training in practical
settings. Indeed, (1) the linearization theory requires small learning
rates or specific network parameterizations (such as the NTK
parameterization), yet in practice a large (initial) learning rate is
typically required in order to reach a good performance; (2) the
linearization theory requires a high width in order for the linearized
model to fit the training dataset and generalize, yet it is unclear
whether the finite-width linearization of practically sized network
have such capacities. Such a gap between linearized and full neural
network training has been identified in recent
work~\cite{chizat2019lazy,ghorbani2019linearized,
  ghorbani2019limitations,li2019towards}, and suggests the need for a
better model towards understanding neural network training in
practical regimes.

% Several concrete questions can be asked towards closing this gap:
% (1) does linearized training approximate full neural net training
% under the standard parameterization? (2) can we ``upgrade'' linearized
% training to a set of approxiamte dynamics that approximates standard
% neural net trainng better, yet still amenable to theoretical analysis?
% (3) is there a regime in which the approximation power of linearized
% trainig and our proposed training dynamics exhibit a significant gap?

Towards closing this gap, in this paper we propose and study
\emph{Taylorized training}, a principled generalization of linearized
training. For any neural network $f_\theta(x)$ and a given
initialization $\theta_0$, assuming sufficient smoothness, we can
expand $f_\theta$ around $\theta_0$ to the $k$-th order for any $k\ge
1$: 
\begin{align*}
  f_\theta(x) = \underbrace{f_{\theta_0}(x) + \sum_{j=1}^k
  \frac{\grad^j_\theta 
  f_{\theta_0}(x)}{j!}[\theta - \theta_0]^{\otimes j}}_{\defeq
  f^{(k)}_{\theta;\theta_0}(x)} + o\paren{\norm{\theta - \theta_0}^k}.
  \\ 
\end{align*}
The model $f^{(k)}_{\theta;\theta_0}(x)$ is exactly the linearized
model when $k=1$, and becomes $k$-th order polynomials of
$\theta-\theta_0$ that are increasingly better local approximations of
$f_\theta$ as we increase $k$. Taylorized training refers to training
these Taylorized models $f^{(k)}$ explicitly (and not necessarily
locally), and using it as a tool towards understanding the training of
the full neural network $f_\theta$. The hope with Taylorized training
is to ``trade expansion order with width'', that is, to hopefully
understand finite-width dynamics better by using a higher expansion
order $k$ rather than by increasing the width.
% \yub{How to properly cite here?}

In this paper, we take an empirical approach towards studying
Taylorized training, demonstrating its usefulness in understanding
finite-width full training\footnote{By full training, we mean the usual
  (non-Taylorized) training of the neural networks.}.
Our main contributions can be summarized
as follows:
\begin{itemize}
\item We experiment with Taylorized training on vanilla convolutional
  and residual networks in their practical training regimes (standard
  parameterization + large initial learning rate) on CIFAR-10. We show
  that Taylorized training gives increasingly better approximations of
  the training trajectory of the full neural net as we increase the
  expansion order $k$, in both the parameter space and the function
  space (Section~\ref{section:experiments}). This is not necessarily
  expected, as higher-order Taylorized models are no longer guaranteed
  to give better approximations when parameters travel significantly,
  yet empiricially they do approximate full training better.
\item We find that Taylorized models can significantly close the 
  performance gap between fully trained neural nets and their
  linearized models at finite width. Finite-width linearized networks
  typically has over 40\% worse test accuracy than their fully
  trained counterparts, whereas quartic (4th order) training is only
  10\%-15\% worse than full training under the same setup.
  % Our quadratic models achieve better
  % accuracy than the best known finite-width linearized models, and our
  % quartic (4th order) WideResNet can nearly match a fully trained
  % 4-layer CNN.
\item We demonstrate the potential of Taylorized training as a tool
  for understanding \emph{layer importance}. Specifically,
  higher-order Taylorized training agrees well with full training in
  layer movements, i.e. how far each layer travels from its
  initialization, whereas linearized training does not agree well. %  and
  % has a less adaptive layer movement curve throughout training.
\item We provide a theoretical analysis on the approximation power of
  Taylorized training (Section~\ref{section:theory}). We prove that
  $k$-th order Taylorized training 
  approximates the full training trajectory with error bound
  $O(m^{-k/2})$ on a wide two-layer network with width $m$. This
  extends existing results on linearized training and provides a
  preliminary justification of our experimental findings.
  % \item We propose specific parameterizations under which the
%   quadratic model approximates much better than the linearized
%   model. More generally, there exists specific parameterizations
%   under which the $k$-th order Taylorized model approximates the NN
%   training trajectory much better than all $\le k-1$-th order models.
\end{itemize}
\paragraph{Additional paper organization}
We provide preliminaries in Section~\ref{section:prelim}, review
linearized training in Section~\ref{section:linearized}, describe
Taylorized training in more details in
Section~\ref{section:taylorized}, and review additional related work
in Section~\ref{section:related}. Additional experimental results are
reported in Appendix~\ref{appendix:additional-exps}.

% \yub{Add pointer to additional experiments in the appendix.}
\paragraph{A visualization of Taylorized training}
A high-level illustration of our results is provided in
Figure~\ref{figure:cnnthick-trajectory}, which visualizes the training
trajectories of a 4-layer convolutional neural network and its
Taylorized models. Observe that the linearized model struggles to
progress past the initial phase of training and is a rather poor
approximation of full training in this setting, whereas higher-order
Taylorized models approximate full training significantly better.

\begin{figure}
  \centering
  \includegraphics[width=0.5\textwidth]{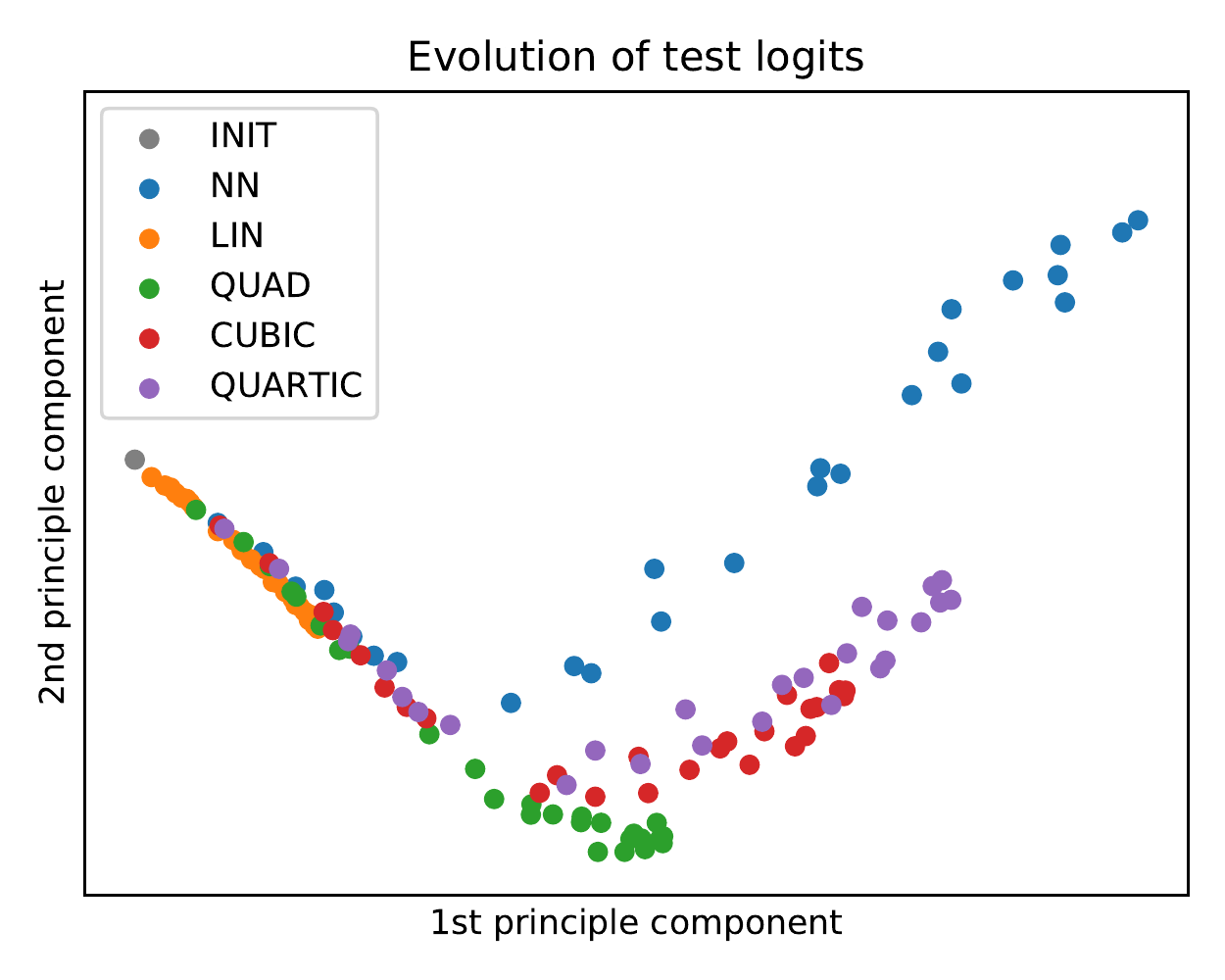}
  \caption{{\bf Function space dynamics of a 4-layer CNN under
    full training ({\bf \color{C0} NN}) and Taylorized training of
    order 1-4 ({\bf \color{C1} linearized}, {\bf \color{C2}
      quadratic}, {\bf \color{C3} cubic}, {\bf \color{C4}
      quartic})}. All models are trained on CIFAR-10 with the same
    initialization + optimization setup, and we plot the
    test logits on the trained models in the first 20 epochs. Each
    point is a 2D PCA embedding of the test logits of the corresponding
    model. Observe that Taylorized training becomes an increasingly better
    approximation of full NN training as we increase the expansion
    order $k$.}
  \label{figure:cnnthick-trajectory}
\end{figure}
\section{Preliminaries}
\label{section:prelim}

% \paragraph{Problem setup}
We consider
% training a neural network $f_\theta:\R^d\to\R^C$ for
% solving
the supervised learning problem
\begin{align*}
  \minimize ~~ L(\theta)=\E_{(x,y)\sim P}[\ell(f_\theta(x), y)],
\end{align*}
where $x\in\R^d$ is the input, $y\in\mc{Y}$ is the label,
$\ell:\R^C\times\mc{Y}\to\R$ is a convex loss function,
$\theta\in\R^p$ is the learnable parameter, and $f_\theta:\R^d\to\R^C$
is the neural network that maps the input to the output
(e.g. the prediction in a regression problem, or the vector of logits in
a classification problem).

This paper focuses on the case where $f_\theta$ is a (deep) neural
network. % , including feedforward, convolutional, and residual.
A standard feedforward neural network with $L$ layers is defined
through $f_\theta(x)=h^L$, where $h^0=x^0=x$, and 
\begin{align}
  \label{equation:standard-param}
  h^{\ell+1} = W^{\ell+1} x^\ell +
  b^{\ell+1},~~~x^{\ell+1}=\sigma(h^{\ell+1})
\end{align}
for all $\ell\in\set{0,\dots,L-1}$, where
$W^\ell\in\R^{d_{\ell}\times d_{\ell-1}}$ are weight matrices,
$b^\ell\in\R^{d_{\ell}}$ are biases, and $\sigma:\R\to\R$ is an
activation function (e.g. the ReLU) applied entry-wise.  We will not
describe other architectures in detail; for the purpose of
describing our approach and empirical results, it suffices to think of
$f_\theta$ as a \emph{general nonlinear function} of the parameter
$\theta$ (for a given input $x$.)

Once the architecture is chosen, it remains to define an
initialization strategy and a learning rule.
\paragraph{Initialization and training}
We will mostly consider
the \emph{standard initialization} (or variants of it such as
Xavier~\cite{glorot2010understanding} or Kaiming~\cite{he2015delving})
in this paper, which for a feedforward network is defined as
\begin{align*}
  & \theta_0 = \set{(W^{\ell}_0, b^{\ell}_0)}_{\ell=1}^{L+1},~{\rm
    where} \\
  & W^{\ell}_{0,ij} \sim \normal(0, 1/d_\ell)~~~{\rm
    and}~~~b^\ell_{0,i}\sim\normal(0, 1),
\end{align*}
and can be similarly defined for convolutional and residual networks.
This is in contrast with the NTK
parameterization~\citep{jacot2018neural}, which encourages the weights
to move significantly less.

We consider training the neural network via (stochastic) gradient
descent:
\begin{align}
  \label{equation:full-training}
  \theta_{t+1} = \theta_t - \eta_t \grad L(\theta_t).
\end{align}
We will refer to the above as \emph{full training} of neural networks,
so as to differentiate with various approximate training regimes to be
introduced below.

% simplicity of presentation, we will describe our approach in terms
% of feedforward networks but use it on all architectures; we
% refer the readers to \yub{cite} for the precise definition of the
% convolutional and residual architectures. 
\section{Linearized Training and Its Limitations}
\label{section:linearized}

We briefly review the theory of linearized
training~\citep{lee2019wide,chizat2019lazy} for explaining the
training and generalization success of neural networks, and provide
insights on its limitations. % as well as potential ways for improving
% upon it.

\subsection{Linearized training and Neural Tangent Kernels}
The theory of linearized training begins with the observation that a
neural network near init can be accurately approximated by a
\emph{linearized} network. Given an initialization $\theta_0$ and an
arbitrary $\theta$ near $\theta_0$, we have that
\begin{align*}
  f_\theta(x) = \underbrace{f_{\theta_0}(x) + \<\grad f_{\theta_0}(x), \theta -
  \theta_0\>}_{\defeq f^{\rm lin}_\theta(x)} + o\paren{\norm{\theta -
  \theta_0}},
\end{align*}
that is, the neural network $f_\theta$ is approximately equal to the
linearized network $f^{\rm lin}_\theta$.
Consequently, near
$\theta_0$, the trajectory of minimizing $L(\theta)$ 
can be well approximated by the trajectory of \emph{linearized
  training}, i.e. minimizing
\begin{align*}
  L^{\rm lin}(\theta) \defeq \E_{(x,y)\sim P}[\ell(f^{\rm
  lin}_\theta(x), y)],
\end{align*}
which is a convex problem and enjoys 
convergence guarantees.

Furthermore, linearized training can approximate
% not only the local behavior, but also
the entire trajectory of full training provided that we are
in a certain \emph{linearized regime} in which we use
\begin{itemize}
\item Small learning rate, so that $\theta$ stays in a small
  neighborhood of $\theta_0$ for any fixed amount of time;
\item Over-parameterization, so that such a neighborhood gives a
  function space that is rich enough to contain a point $\theta$ so
  that $f^{\rm lin}_\theta$ can fit the entire training dataset.
\end{itemize}
%   overparameterized} (wide) network and use a {\bf small learning
%   rate}, so that $\theta$ indeed stays in a small neighborhood of
% $\theta_0$, but the neighborhood is rich enough (because of the width)
% to contain a point $\theta$ such that $f^{\rm lin}_\theta$ fits the
% entire training dataset.
As soon as we are in the above linearized regime,
gradient descent is guaranteed to reach a global minimum
~\citep{du2019gradient,allen2019convergence,zou2019gradient}.
Further, as the width goes to infinity, due to randomness in the
initialization $\theta_0$, the function space containing such
linearized models goes to a statistical limit governed by the Neural
Tangent Kernels (NTKs)~\cite{jacot2018neural}, so that wide networks
trained in this linearized regime generalize as well as a kernel
method~\cite{arora2019fine,allen2019learning}.

\subsection{Unrealisticness of linearized training in practice}
%Though powerful, the theory of linearized training does not yet
%explain the full picture of neural network training---a number of
%limitations of linearized training have indeed been identified in both
%theory and practice~\yub{maybe cite?}.

Our key concern about the theory of linearized training
% which this
% paper will focus on, is the mismatch in training regimes---
is that there are significant differences between training regimes in
which the linearized approximation is accurate, and regimes in which
neural nets typically attain their best performance in practice. More
concretely,
\begin{enumerate}[(1)]
\item Linearized training is a good approximation of full training
  under {\bf small learning rates}\footnote{Or large learning rates
    but on the NTK parameterization~\cite{lee2019wide}.} in which each
  individual weight barely moves. However,
  % it is frequently observed that
  neural networks typically attain their best test performance when using a
  {\bf large
    (initial) learning rate}, in which the weights move significantly
  in a way not explained by linearized training~\cite{li2019towards};
\item Linearized networks are powerful models on their own when the
  base architecture is {\bf over-parameterized}, but can be rather
  poor when the network is of a {\bf practical
    size}. Indeed, infinite-width linearized models such as CNTK
  achieve competitive performance on benchmark
  tasks~\cite{arora2019exact,arora2019harnessing}, yet their
  finite-width counterparts often perform significantly worse
  % than
  % both the infinite-width versions and the fully-trained neural
  % nets
  ~\cite{lee2019wide,chizat2019lazy}.
\end{enumerate}
\section{Taylorized Training}
\label{section:taylorized}
Towards closing this gap between linearized and full training, we
propose to study \emph{Taylorized training}, a
principled extension of linearized training. Taylorized training
involves training \emph{higher-order expansions} of the neural network
around the initialization. For
any $k\ge 1$---assuming sufficient smoothness---we can Taylor expand
$f_\theta$ to the $k$-th order as
\begin{align*}
  f_{\theta}(x) = \underbrace{f_{\theta_0}(x) + \sum_{j=1}^k
  \frac{\grad^j f_{\theta_0}(x)}{j!}[\theta - \theta_0]^{\otimes
  j}}_{\defeq f^{(k)}_{\theta}(x)} + o\paren{\norm{\theta -
  \theta_0}^k},
\end{align*}
where we have defined the $k$-th order Taylorized model
$f^{(k)}_{\theta}(x)$. % \footnote{We will drop the
% dependence on $\theta_0$ when it's clear from the context.}.
The Taylorized model $f^{(k)}_\theta$ reduces to the linearized model when
$k=1$, and is a $k$-th order polynomial model for a general $k$, where
the ``features'' are
$\grad^j f_{\theta_0}(x)\in\R^{C\times p^{\otimes j}}$ (which depend
on the architecture $f$ and initialization $\theta_0$), and the
``coefficients'' are $(\theta-\theta_0)^{\otimes j}$ for
$j=1,\dots,k$.

Similar as linearized training, we
define Taylorized training as the process (or trajectory) for training
$f^{(k)}_\theta$ via gradient descent, starting from the
initialization $\theta=\theta_0$. Concretely, the trajectory
for $k$-th order Taylorized training will be denoted as
$\theta^{(k)}_t$, where % (for the case of gradient descent)
\begin{equation}
  \label{equation:taylorized-training}
  \begin{aligned}
    & \theta^{(k)}_{t+1} = \theta^{(k)}_t - \eta_t \grad
    L^{(k)}(\theta^{(k)}_t),~~{\rm where} \\
    & L^{(k)}(\theta) \defeq \E_{(x,y)\sim P}[\ell(f^{(k)}_\theta(x), y)].
  \end{aligned}
\end{equation}

Taylorized models arise from a similar principle as linearized models
(Taylor expansion of the neural net), and gives increasingly better
approximations of the neural network (at least locally) as we increase
$k$. Further, higher-order Taylorized training ($k\ge 2$) are no
longer convex problems, yet they model the non-convexity of full
training in a mild way that is potentially amenable to theoretical
analyses. Indeed, quadratic training ($k=2$) been shown to enjoy a
nice optimization landscape and achieve better sample complexity than
linearized training on learning certain simple
functions~\cite{bai2020beyond}. Higher-order training also has the
potential to be understood through its polynomial structure and its
connection to tensor decomposition
problems~\cite{mondelli2019connection}.

\paragraph{Implementation}
Naively implementing Taylorization by directly computing higher-order
derivative tensors of neural networks is prohibitive in both memory
and time. Fortunately, Taylorized models can be efficiently
implemented through a series of nested Jacobian-Vector Product
operations (JVPs). Each JVP operation can be computed with the
$\mathcal{R}$-operator algorithm of \citet{pearlmutter1994fast}, which
gives directional derivatives through arbitrary differentiable
functions, and is the transpose of backpropagation.

For any function
$f$ with parameters $\theta_0$, we denote its JVP with respect to the
direction $\Delta\theta\defeq \theta-\theta_0$ using the notation of \citet{pearlmutter1994fast} by  
\begin{equation}
  \mathcal{R}_{\Delta \theta}(f_{\theta_0}(x)) \defeq
  \frac{\partial}{\partial r}f_{(\theta_0+r\Delta \theta)}(x)|_{r=0}.
\end{equation}

% We use the notation $\mathcal{R}^k_{\Delta \theta}(f(\theta_0))$ to describe $k$ nested applications of the $\mathcal{R}$-operator, giving the $k$-th order term in the Taylor expansion, where 
% \begin{equation}
%      \brac{\mathcal{R}^k_{\Delta \theta}(f)} (\theta_0)=\left\{
%     \begin{aligned}
%         & \mathcal{R}_{\Delta \theta}(f(\theta_0))~~~{\rm for}~k=1, \\
%         & \mathcal{R}_{\Delta \theta}(\mathcal{R}^{k-1}_{\Delta \theta}(f(\theta_0)))~~~{\rm for}~k\ge 2,
%     \end{aligned}
%     \right.
% \end{equation}
% and nested $\mathcal{R}$-operator functions are given by
% \begin{equation}
%  \mathcal{R}_{\Delta \theta}(\mathcal{R}^{k}_{\Delta \theta}(f(\theta_0))) = \frac{\partial}{\partial r} \mathcal{R}^{k}_{\Delta \theta}(f(\theta_0+r{\Delta \theta}))|_{r=0}.
% \end{equation}

The $k$-th order Taylorized model can be computed as
\begin{equation}
  f^{(k)}_\theta(x) = f_{\theta_0}(x) + \sum_{j=1}^k
  \frac{\mathcal{R}^{j}_{\Delta\theta}(f_{\theta_0}(x))}{j!},
\end{equation} 
where $\mc{R}^j$ is the $j$-times nested evaluation of the
$\mc{R}$-operator. 

Our implementation uses {\tt
  Jax}~\cite{bradbury2018jax} and {\tt
  neural\_tangents}~\cite{novak2020neural} which has built-in support
for Taylorizing any function to an arbitrary order based on nested
JVP operations.

\section{Experiments}
\label{section:experiments}

\begin{table*}
  \small
  \centering
  \begin{tabular}{l@{\ \ }|@{\ \ }c@{\ \ }c@{\ \ }| c@{\ \ }c@{\ \ }c@{\ \ }c@{\ \ }c@{\ \ }c@{\ \ } c@{\ \ } @{\ \ }c@{\ \ }c@{\ \ }}
    \toprule
    Name & Architecture & Params& Train For & Batch & Accuracy & Opt &
                                                                       Rate
    & Grad Clip & LR Decay Schedule \\ 
    \midrule
    CNNTHIN & CNN-4-128 & 447K & 200 epochs & 256 & 81.6\%
                                                          & SGD & 0.1
    & 5.0
                                                                       
    & 10x drop at 100, 150 epochs \\
   CNNTHICK & CNN-4-512 & 7.10M & 160 epochs & 64 & 85.9\%
                                                          & SGD &  0.1
    & 5.0
    & 10x drop
    at 80, 120 epochs\\
    \midrule
    WRNTHIN & WideResNet-16-4-128 & 3.22M & 200 epochs & 256 & 88.1\% &
                                                                  SGD
                                                                     &
                                                                       1e-1.5
    & 10.0 & 10x drop at 100, 150 epochs \\ 
    WRNTHICK & WideResNet-16-8-256 & 12.84M & 160 epochs & 64 & 91.7\%
                                                               & SGD & 
                                                                    1e-1.5
    & 10.0 & 10x drop at 80, 120 epochs \\ 
    \bottomrule
  \end{tabular}
  \caption{{\bf Our architectures and training setups.} CNN-$L$-$C$ stands for a
  CNN with depth $L$ and $C$ channels per
  layer. WideResNet-$L$-$k$-$C$ stands for a WideResNet with depth $L$, widening
  factor $k$, and $C$ channels in the first convolutional
  layer.}
  \label{table:architectures}
\end{table*}

We experiment with Taylorized training on convolutional and residual
networks for the image classification task on CIFAR-10. % Our
% experiments demonstrate the advantage of Taylorized training in terms
% of both approximating full training and reaching good test performance
% in practical training regimes.

\subsection{Basic setup}
We choose four representative architectures for the image
classification task: two CNNs with 4 layers + Global Average Pooling
(GAP) with width $\set{128, 512}$, and two
WideResNets~\cite{zagoruyko2016wide} with depth 16 and different
widths as well. All networks use standard parameterization and are
trained with the cross-entropy loss\footnote{Different from prior work
  on linearized training which primarily focused on the squared
  loss~\cite{arora2019exact,lee2019wide}.}. We optimize the training
loss using SGD with a large initial learning rate + learning rate
decay.\footnote{We also use gradient clipping with a large clipping
  norm in order to prevent occasional gradient blow-ups.}

% \yub{Highlight we use practical training regimes? (large lr and
%   standard parameterization.)}
For each architecture, the
initial learning rate was tuned within
$\set{10^{-2}, 10^{-1.5}, 10^{-1}, 10^{-0.5}}$ and chosen to be the
largest learning rate under which the full neural network can stably
train (i.e. has a smoothly decreasing training loss). We use standard
data augmentation (random crop, flip, and standardize) as a
optimization-independent way for
improving generalization.
% \footnote{As opposed
%   to weight decay which can potentially twist the training trajectory
% differently for full and Taylorized training.}
Detailed training
settings for each architecture are summarized in
Table~\ref{table:architectures}.

% \paragraph{Training Taylorized models from a common initialization}
\paragraph{Methodology}
For each architecture, we train Taylorized models of order
$k\in\set{1,2,3,4}$ (referred to as \{linearized, quadratic, cubic,
quartic\} models) from the {\bf same initialization} as full training
using the {\bf exact same optimization setting} (including learning
rate decay, gradient clipping, minibatching, and data augmentation
noise). This allows us to eliminate the effects of optimization setup
and randomness, and examine the agreement between Taylorized and full
training in identical settings.

\subsection{Approximation power of Taylorized training}
\label{section:approximation-exp}
% \yub{better subtitle?}
We examine the approximation power of Taylorized training through
comparing Taylorized training of different orders in terms of both the
training trajectory and the test performance.

\paragraph{Metrics}
We monitor the training loss and test accuracy for
both full and Taylorized training. We also evaluate the approximation
error between Taylorized and
full training quantitatively through the following similarity metrics
between models:  % in both the parameter space and the function space:
\begin{itemize}
\item Cosine similarity in the parameter space, defined as
  \begin{equation*}
    \cosparam_t \defeq \frac{\<\theta^{(k)}_t - \theta_0,
      \theta_t - \theta_0\>}{\big\|\theta^{(k)}_t -
        \theta_0\big\|\big\|\theta_t - \theta_0\big\|},
  \end{equation*}
  where (recall~\eqref{equation:taylorized-training}
  and~\eqref{equation:full-training}) $\theta^{(k)}_t$ and
  $\theta_t$ denote the parameters in $k$-th order 
  Taylorized training and full training, and $\theta_0$ is their
  common initialization.
\item Cosine similarity in the function space, defined as
  \begin{equation*}
    \cosfunc_t \defeq \frac{\<f^{(k)}_{\theta^{(k)}_t} - f_{\theta_0},
      f_{\theta_t} -
      f_{\theta_0}\>}{\big\|f^{(k)}_{\theta^{(k)}_t}-
        f_{\theta_0}\big\|\big\|f_{\theta_t} - 
        f_{\theta_0}\big\|},
  \end{equation*}
  where we have overloaded the notation
  $f\in\R^{C\times n_{\rm test}}$ (and similarly $f^{(k)}$) to denote
  the output (logits) of a model on the test dataset\footnote{We
    centralized (demeaned) the logits for each example along the
    classification axis so as to remove the effect of the invariance
    in the softmax mapping.
  }.
\end{itemize}
% For the purpose of comparing models, these two metrics are
% complementary to each other: the function space similarity measures
% distance between models in a sensible way (by comparing test
% predictions), though it can be strongly tied to the performance of the
% model (two high-performing models would output similar logits); the
% parameter space similarity looks at the raw parameter space without
% knowledge about the actual architecture, yet is itself a metric that
% is more independent of the performance of the model.

\begin{figure*}
  \centering
    \includegraphics[width=0.24\textwidth]{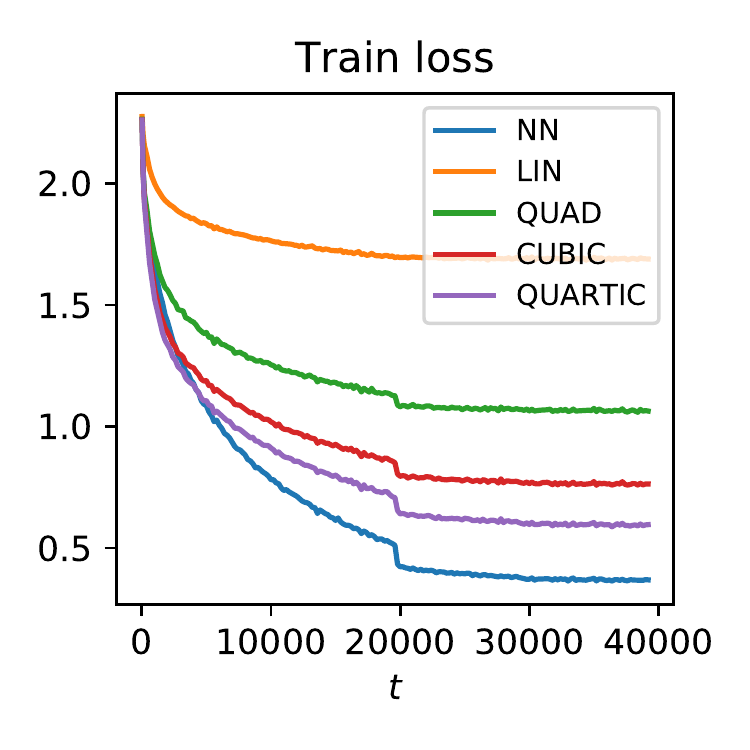}
    \includegraphics[width=0.24\textwidth]{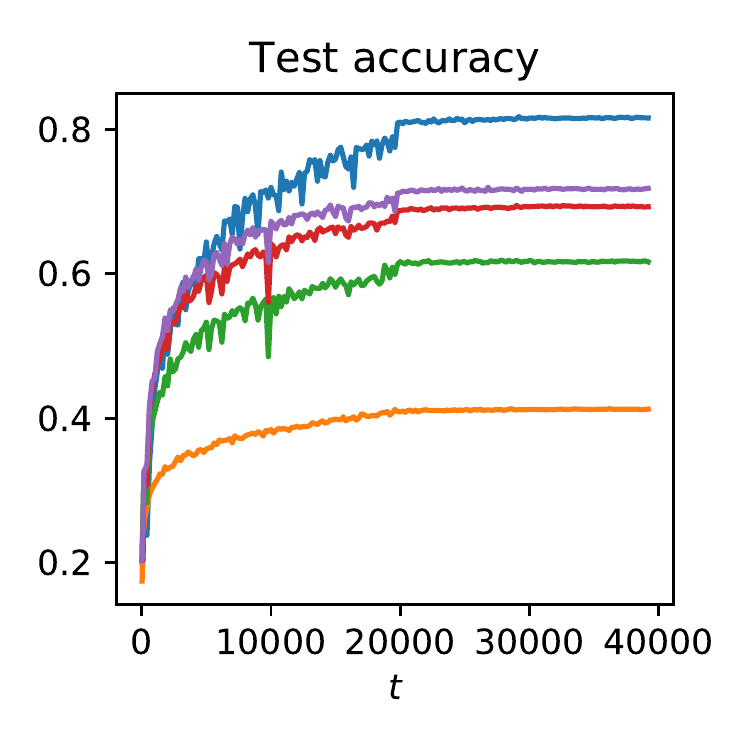}
    \includegraphics[width=0.24\textwidth]{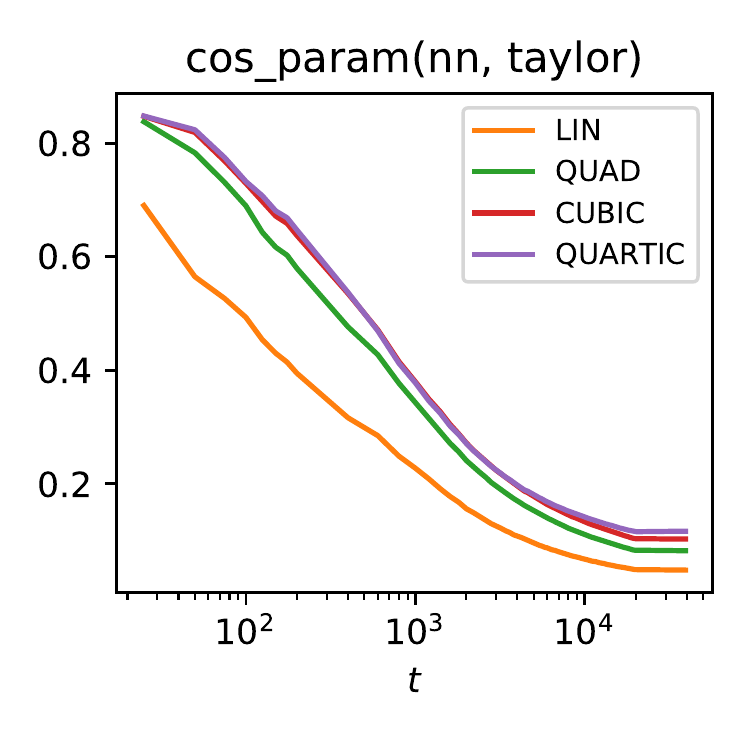}
    \includegraphics[width=0.24\textwidth]{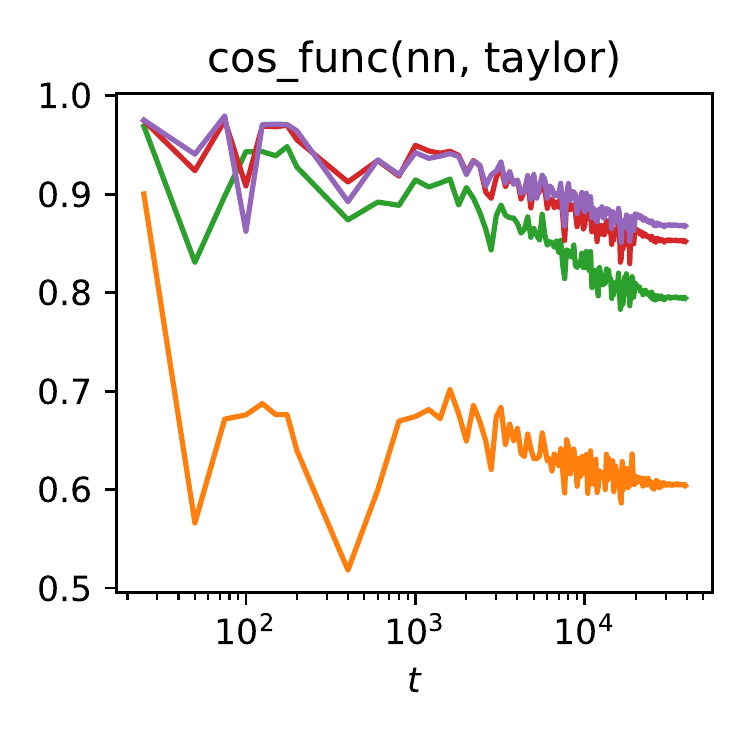}
    \caption{{\bf $k$-th order Taylorized training approximates full
      training increasingly better with $k$ on the CNNTHIN
      model.} Training statistics are plotted for \{\textbf{{\color{C0}
          full}, {\color{C1} linearized}, {\color{C2} quadratic},
        {\color{C3} cubic}, {\color{C4} quartic}}\} models. Left to
      right: (1) training loss; (2) test accuracy; (3) cosine
      similarity between Taylorized and full training in the parameter
      space; (4) cosine similarity between Taylorized and full
      training in the function (logit) space. All models are trained on
      CIFAR-10 for 39200 steps, and a 10x learning rate decay happened
      at step \{19600, 29400\}.}
  \label{figure:cnnthin-approximation}
\end{figure*}

\begin{table*}
  \centering
  \begin{tabular}{c|c|c|c|c}
    \hline
    Architecture & CNNTHIN & CNNTHICK & WRNTHIN & WRNTHICK \\ 
    \hline
    Linearized ($k=1$) & 41.3\% & 49.0\% & 50.2\% & 55.3\% \\
    Quadratic ($k=2$) & 61.6\% & 70.1\% & 65.8\% & 71.7\% \\
    Cubic ($k=3$) & 69.3\% & 75.3\% & 72.6\% & 76.9\% \\
    Quartic ($k=4$) & 71.8\% & 76.2\% & 75.6\% & 78.7\% \\
    \hline
    Full network & 81.6\% & 85.9\% & 88.1\% & 91.7\% \\
    \hline
  \end{tabular}
  \caption{{\bf Final test accuracy on CIFAR-10 for Taylorized models
    trained under the same optimization setup as full neural
    nets.} Details about the architectures and their training setups
    can be found in Table~\ref{table:architectures}.}
  \label{table:taylorized-performance}
\end{table*}

\vspace{-0.3cm}
\paragraph{Results}
Figure~\ref{figure:cnnthin-approximation} plots training and
approximation metrics for full and Taylorized training on the CNNTHIN
model.
% It can be seen that
% Taylorized models are good interpolators of linearized and full
% training --
Observe that Taylorized models are much better approximators than
linearized models in both the parameter space and the function
space---both cosine similarity curves shift up as we increase $k$ from
1 to 4. Further, for the cubic and quartic 
models, the cosine similarity in the logit space stays above 0.8 over
the entire training trajectory (which includes both weakly and
strongly trained models), suggesting a fine agreement between
higher-order Taylorized training and full training. Results
for \{CNNTHICK, WRNTHIN, WRNTHICK\} are qualitatively similar and are
provided in Appendix~\ref{appendix:approximation}.

We further report the final test performance of the Taylorized models
on all architectures in Table~\ref{table:taylorized-performance}. We
observe that
% \vspace{-0.1cm}
\begin{enumerate}[(1)]
\item Taylorized models can indeed close the performance gap between
  linearized and full training: linearized models are typically
  30\%-40\% worse than fully trained networks, whereas quartic (4th
  order Taylorized) models are within \{10\%, 13\%\} of a fully trained
  network on \{CNNs, WideResNets\}.
\item All Taylorized models can benifit from increasing the width
  (from CNNTHIN to CNNTHICK, and WRNTHIN to WRNTHICK), but the
  performance of higher-order models ($k=3,4$) are generally less
  sensitive to width than lower-order models ($k=1,2$), suggesting
  their realisticness for explaining the training behavior of
  practically sized finite-width networks.
\end{enumerate}

\paragraph{On finite- vs. infinite-width linearized models}
We emphasize that the performance of our baseline linearized models in
Table~\ref{table:taylorized-performance}  (40\%-55\%) is at finite
width, and is thus not directly comparable to
existing results on infinite-width linearized models such as
CNTK~\cite{arora2019exact}. %  or its enhanced
% versions~\cite{li2019enhanced}
It is possible to achieve stronger results with finite-width
linearized networks by using the NTK parameterization, which more
closely resembles the infinite width limit. However, full neural net
training with this re-parameterization results in significantly
weakened performance, suggesting its unrealisticness. The
best documented test accuracy of a finite-width linearized network on
CIFAR-10 is $\sim$65\%~\cite{lee2019wide}, and due to the NTK
parameterization, the neural network trained under these same settings
only reached $\sim$70\%. In contrast, our best higher order models can
approach 80\%, and are trained under realistic settings where a neural
network can reach over 90\%. 

%Indeed, it has been observed that finite-width linearized networks are hard to optimize (e.g. to nearly 0 training loss) because of poorconditioning~\cite{chizat2019lazy}. 

\subsection{Agreement on layer movements}
% \yub{better subtitle?}

Layer importance, i.e. the contribution and importance of each layer
in a well-trained (deep) network, has been identified as a useful
concept towards building an architecture-dependent understanding on
neural network training~\cite{zhang2019all}. 
Here we demonstrate that
higher-order Taylorized training has the potential to lead to better
understandings on the layer importance in full training.

\paragraph{Method and result} We examine \emph{layer movements},
i.e. the distances each layer has travelled along training, and
illustrate it on both full and Taylorized
training.\footnote{Taylorized models $f^{(k)}_\theta$ are polynomials
  of $\theta$ where $\theta$ has the same shape as the base
  network. By a ``layer'' in a Taylorized model, we mean the partition
  that's same as how we partition $\theta$ into layers in the base
  network.}
In Figure~\ref{figure:layer-movement}, we plot the layer movements on
the CNNTHIN and WRNTHIN models. Compared with linearized training,
quartic training agrees with full training much better in the shape of
the layer movement curve, both at an early stage and at
convergence. Furthermore, comparing the layer movement curves between the
10th epoch and at convergence, quartic training seems to be able to
adjust the shape of the movement curve much better than linearized
training. % , suggesting their advantage for understanding layer
% importance.

% \paragraph{More on layer importance}
Intriguing results about layer importance has also been (implicitly)
shown in the study of infinite-width linearized models (i.e. NTK type
kernel methods). For example, it has been observed that the CNN-GP
kernel (which corresponds to training the top layer only) has
consistently better generalization performance than the CNTK kernel
(which corresponds to training all the
layer)~\cite{li2019enhanced}. In other words, when training an
extremely wide convolutional net on a finite dataset, training the
last layer only gives a better generalization performance (i.e. a
better implicit bias); existing theoretical work on linearized
training fall short of understanding layer importance in these
settings. We believe Taylorized training can serve as an (at
least empirically) useful tool towards understanding layer importance.

\begin{figure*}[t]
  \centering
  \begin{subfigure}[t]{0.48\textwidth}
    \centering
    \includegraphics[width=\textwidth]{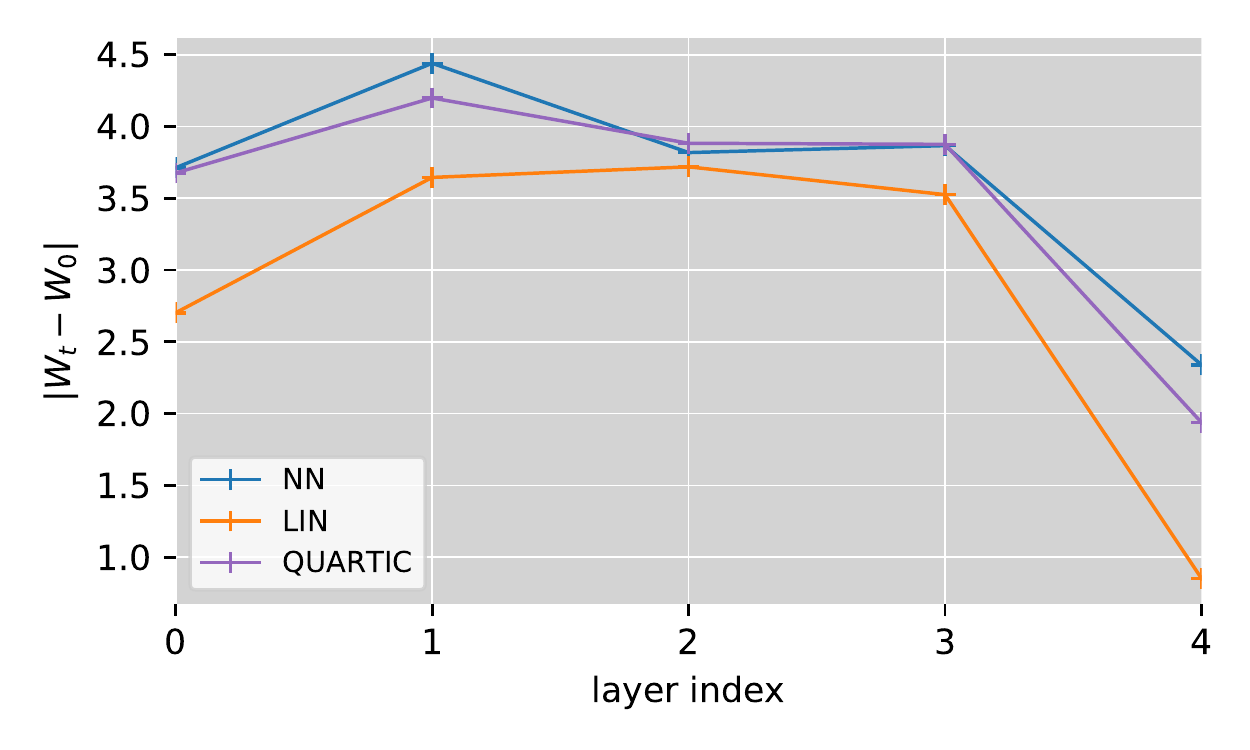}
    \caption{CNNTHIN at 10 epochs}
  \end{subfigure}
  \begin{subfigure}[t]{0.48\textwidth}
    \centering
    \includegraphics[width=\textwidth]{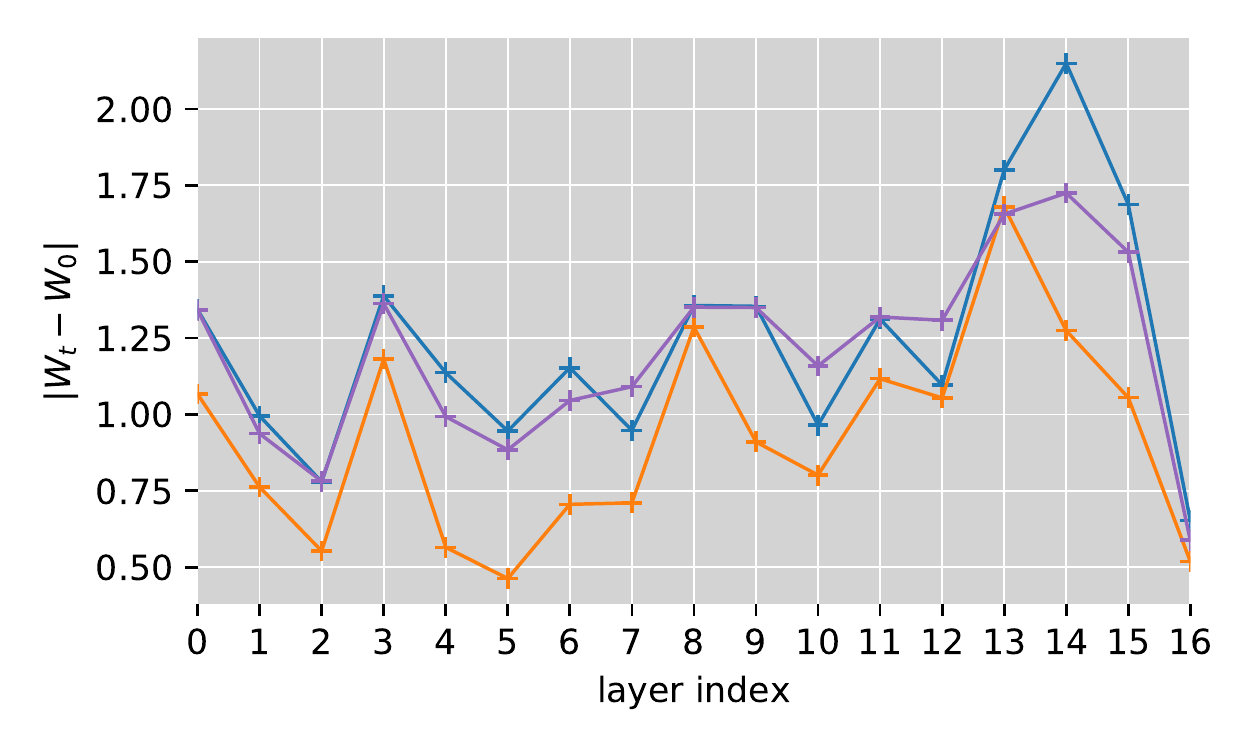}
    \caption{WRNSMALL at 10 epochs}
  \end{subfigure}
  \begin{subfigure}[t]{0.48\textwidth}
    \centering
    \includegraphics[width=\textwidth]{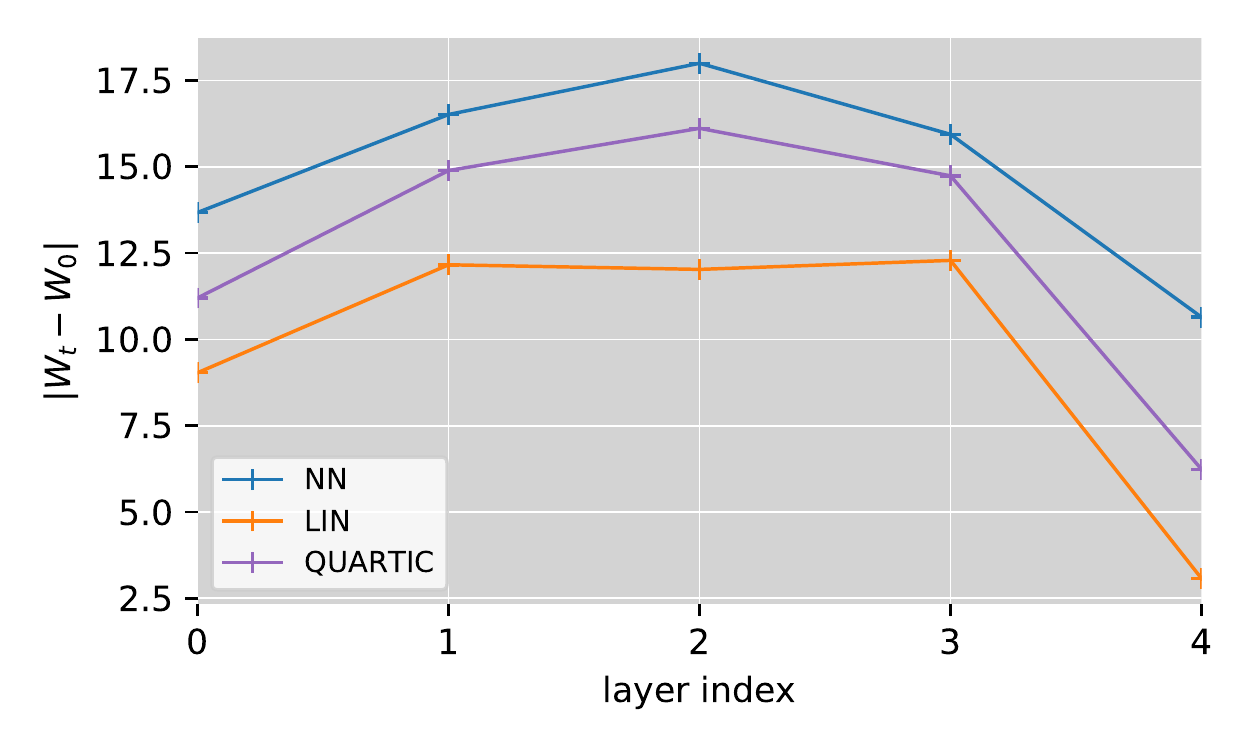}
    \caption{CNNTHIN at convergence}
  \end{subfigure}
  \begin{subfigure}[t]{0.48\textwidth}
    \centering
    \includegraphics[width=\textwidth]{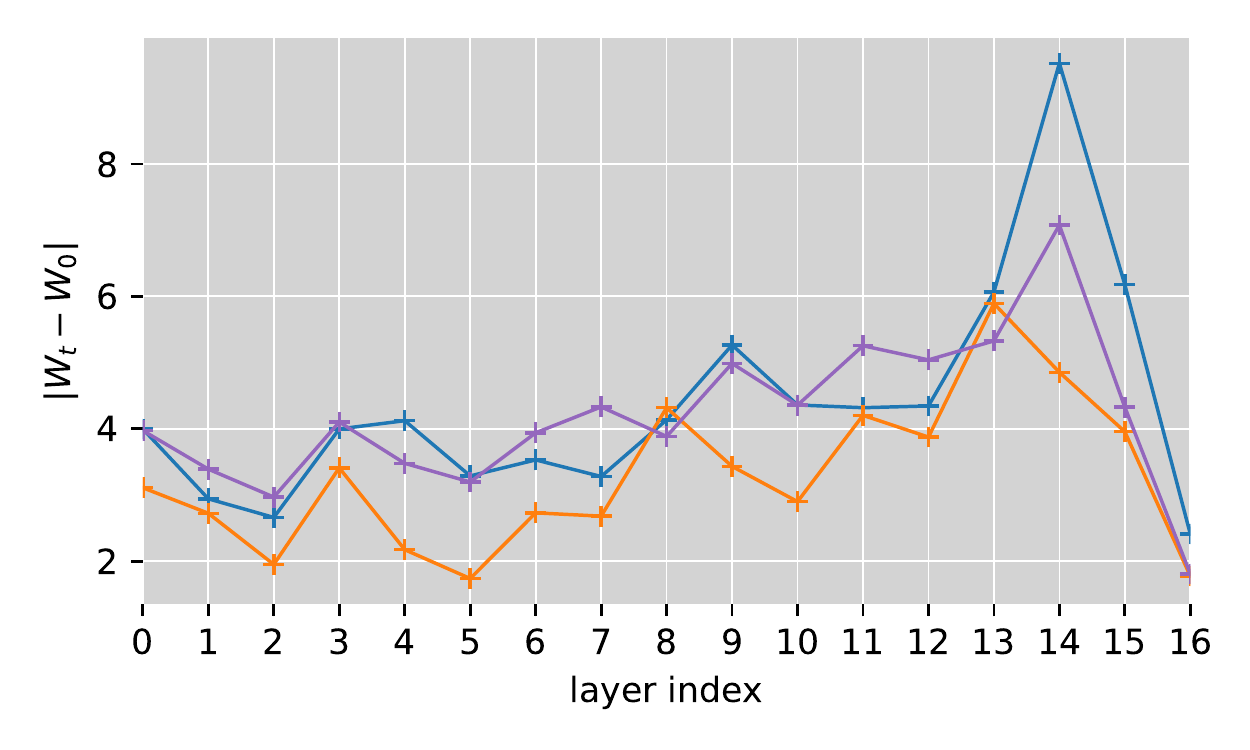}
    \caption{WRNSMALL at convergence}
  \end{subfigure}
  \caption{{\bf Layer movement of {\bf \color{C0} full NN},
    {\bf \color{C1} linearized}, and {\bf \color{C4} quartic}
    models.} Compared with linearized training, quartic (4th order)
    Taylorized training agrees with the full neural network much
    better in terms of layer movement, both at the initial stage and
    at convergence.}
  \label{figure:layer-movement}
\end{figure*}
\section{Theoretical Results}
\label{section:theory}
% Towards understanding Taylorized training, 
We provide a theoretical analysis on the distance between the
trajectories of Taylorized training and full training on wide neural
networks. 

% \yub{Need to change to NTK parameterization. Standard param does not
%   give a well-defined NTK on two-layer nets.}

\paragraph{Problem Setup}
We consider training a wide two-layer neural network with width $m$
and the NTK parameterization\footnote{For wide two-layer networks,
  non-trivial linearized/lazy training can only happen at the NTK
  parameterization; standard parameterization + small learning rate
  would collapse to training a linear function of the input.}:
\begin{equation}
  \label{equation:two-layer-net}
  f_W(x) = \frac{1}{\sqrt{m}} \sum_{r=1}^m a_r\sigma(w_r^\top
  x),
\end{equation}
where $x\in\R^d$ is the input satisfying $\ltwo{x}=1$,
$\set{w_r}_{r\in[m]}\subset\R^d$ are the neurons,
$\set{a_r}\subset \R$ are the top-layer coefficients, and
$\sigma:\R\to\R$ is a smooth activation function.  We set $\set{a_r}$
fixed and only train $\set{w_r}$, so that the learnable parameter of
the problem is the weight matrix
$W=[w_1,\dots,w_m]\in\R^{d\times m}$.\footnote{Setting the top layer
  as fixed is standard in the analysis of two-layer networks in the
  linearized regime, see e.g.~\cite{du2018gradient}.}

We initialize $(a,W)=(a_0, W_0)$ randomly according to the standard
initialization, that is,
\begin{align*}
  a_r \equiv  a_{r,0}\sim \Unif\set{\pm 1}~~~{\rm
  and}~~~w_{r,0}\sim\normal(0, I_d).
\end{align*}
We consider the regression task over a finite dataset $\mc{D}=\set{(x_i,
y_i)\in\R^ d\times\R:i\in[n]}$ with squared loss
\begin{equation*}
  L(W) = \frac{1}{2}\E_{(x,y)\sim\mc{D}}\brac{ (y - f_W(x))^2 },
\end{equation*}
and train $W$ via gradient flow (i.e. continuous time gradient
descent) with ``step-size''\footnote{Gradient flow trajectories are
  invariant to the step-size choice; however, we choose a
  ``step-size'' so as to simplify the presentation.}  $\eta_0$:
\begin{equation}
  \label{equation:full-training-gf}
  \dot{W}_t = -\eta_0\grad L(W_t).
\end{equation}

\paragraph{Taylorized training}
We compare the full training dynamics~\eqref{equation:full-training-gf}
with the corresponding Taylorized training dynamics. The $k$-th order
Taylorized model for the neural
network~\eqref{equation:two-layer-net}, denoted as $f^{(k)}_W$, has the
form
\begin{equation*}
  f^{(k)}_W(x) = \frac{1}{\sqrt{m}}\sum_{r=1}^m a_r \sum_{j=0}^k
  \frac{\sigma^{(j)}(w_{r,0}^\top x)}{j!} \brac{(w_r - w_{r,0})^\top
    x}^j.
\end{equation*}
The Taylorized training dynamics can be described as
\begin{equation}
  \label{equation:taylorized-training-gf}
  \begin{aligned}
    & \dot{W}^{(k)}_t = -\eta_0 \grad L^{(k)}(W^{(k)}_t),~~~{\rm where} \\
    & L^{(k)}(W) = \frac{1}{2}\E_{(x,y)\sim\mc{D}}\brac{ (y - f^{(k)}_W(x))^2 },
  \end{aligned}
\end{equation}
starting at the same initialization $W^{(k)}_0=W_0$.

We now present our main theoretical result which gives bounds on the
agreement between $k$-th order Taylorized training and full training
on wide neural networks.
\begin{theorem}[Agreement between Taylorized and full training:
  informal version]
  \label{theorem:main}
  There exists a suitable step-size $\eta_0$ such that for any fixed
  $t_0>0$ and all suffiicently large $m$,
  % There exists a sufficiently large $m_0$ and a step-size $\eta_0$
  % such that for all $m\ge m_0$,
  with high probability over the random initialization, % the trajectories of
  full training~\eqref{equation:full-training-gf} and 
  Taylorized training~\eqref{equation:taylorized-training-gf} are coupled
  in both the parameter space and the function space:
  \begin{align*}
    & \sup_{t\le t_0} \lfro{W_t - W^{(k)}_t} \le O(m^{-k/2}), \\
    & \sup_{t\le t_0} \abs{f_{W_t}(x) - f^{(k)}_{W^{(k)}_t}(x)} \le
      O(m^{-k/2})~~\textrm{for all}~\ltwo{x}=1.
  \end{align*}
\end{theorem}
% \paragraph{Proof sketch and implications}
% \yub{check whether results holds for all $x$.}
% The proof of Theorem~\ref{theorem:main} builds on~\yub{TBA.} 

Theorem~\ref{theorem:main} extends existing results which state that
linearized training approximates full training with an error bound
$O(1/\sqrt{m})$ in the function
space~\cite{lee2019wide,chizat2019lazy}, showing that higher-order
Taylorized training enjoys a stronger approximation bound
$O(m^{-k/2})$. Such a bound corroborates our experimental finding
that Taylorized training are increasingly better approximations of
full training as we increase $k$. We defer
the formal statement of Theorem~\ref{theorem:main} and its 
proof to Appendix~\ref{appendix:proof-main}.

We emphasize that Theorem~\ref{theorem:main} is still only mostly
relevant for explaining the initial stage rather than the entire
trajectory for full training in practice, due to the fact that the
result holds for gradient flow which only simulates gradient descent
with an infinitesimally small learning rate.  We believe it is an
interesting open direction how to prove the coupling between neural
networks and the $k$-th order Taylorized training under large learning
rates.

% {\bf Remark}~\yub{Extends existing result over linearized training
%   (cite). Error bounds decays as $m^{-k}$ can be interpreted as works
%   better at finite width.}~\yub{Prove coupling out-of the NTK regime
%   is left as future work.}
\section{Related Work}
\label{section:related}
Here we review some additional related work.

\paragraph{Neural networks, linearized training, and kernels}
The connection between wide neural networks and kernel methods has
first been identified in~\cite{neal1996priors}.
A fast-growing body of recent work has studied the interplay between
wide neural networks, linearized models, and their infinite-width
limits governed by either the Gaussian Process (GP) kernel (corresponding
to training the top linear layer
only)~\cite{daniely2017sgd,lee2018deep,matthews2018gaussian} or the 
Neural Tangent Kernel (corresponding to training all the
layers)~\cite{jacot2018neural}. By exploiting such an interplay, it
has been shown that gradient descent on overparameterized neural nets
can reach global
minima~\cite{jacot2018neural,du2018gradient,du2019gradient,allen2019convergence,zou2019gradient,lee2019wide},  
and generalize as well as a kernel
method~\cite{li2018learning,arora2019fine,cao2019generalization}.

\paragraph{NTK-based and NTK-inspired learning algorithms}
Inspired by the connection between neural nets and kernel methods,
algorithms for computing the \emph{exact} (limiting) GP / NTK kernels
efficiently has been
proposed~\cite{arora2019exact,lee2019wide,novak2020neural,yang2019tensor}
and shown to yield state-of-the-art kernel-based algorithms on
benchmark learning
tasks~\cite{arora2019exact,li2019enhanced,arora2019harnessing}. The
connection between neural nets and kernels have further been used in
designing algorithms for general machine learning use cases such as
multi-task learning~\cite{mu2020gradients} and protecting against
noisy labels~\cite{hu2020simple}.

\paragraph{Limitations of linearized training}
The performance gap between linearized and fully trained networks has
been empirically observed
in~\cite{arora2019exact,lee2019wide,chizat2019lazy}. On the
theoretical end, the sample complexity gap between linearized training
and full training has been shown
in~\cite{wei2019regularization,ghorbani2019limitations,allen2019can,yehudai2019power}
under specific data distributions and architectures.

\paragraph{Provable training beyond linearization}
\citet{allen2019learning,bai2020beyond} show that wide neural nets can
couple with quadratic models with provably nice optimization
landscapes and better generalization than the
NTKs, and~\citet{bai2020beyond} furthers show the sample complexity
benefit of $k$-th order models for all $k$.
~\citet{li2019towards} show that a large initial learning rate +
learning rate decay generalizes better than a small learning rate for
learning a two-layer network on a specific toy data distribution.

An parallel line of work studies over-parameterized neural net
training in the mean-field limit, in which the training dynamics can
be characterized as a PDE over the distribution of
weights~\cite{mei2018mean,chizat2018global,rotskoff2018neural,sirignano2018mean}. Unlike
the NTK regime, the mean-field regime moves weights significantly,
though the inductive bias (what function it converges to) and
generalization power there is less clear. 
\section{Conclusion}
In this paper, we introduced and studied Taylorized training. We
demonstrated experimentally the potential of Taylorized training in
understanding full neural network training, by showing its advantage in
terms of approximation in both weight and function space, training and
test performance, and other empirical properties such as layer
movements. We also provided a preliminary theoretical analysis on the
approximation power of Taylorized training.

We believe Taylorized training can serve as a useful tool towards
studying the theory of deep learning and open many interesting
future directions. For example, can we prove the coupling between full
and Taylorized training with large learning rates? How well does
Taylorized training approximate full training as $k$ approaches
infinity? Following up on our layer movement experiments, it would
also be interesting use Taylorized training to study the properties of
neural network architectures or initializations.

\bibliography{bib}

\begin{thebibliography}{41}
\providecommand{\natexlab}[1]{#1}
\providecommand{\url}[1]{\texttt{#1}}
\expandafter\ifx\csname urlstyle\endcsname\relax
  \providecommand{\doi}[1]{doi: #1}\else
  \providecommand{\doi}{doi: \begingroup \urlstyle{rm}\Url}\fi

\bibitem[Allen-Zhu \& Li(2019)Allen-Zhu and Li]{allen2019can}
Allen-Zhu, Z. and Li, Y.
\newblock What can resnet learn efficiently, going beyond kernels?
\newblock In \emph{Advances in Neural Information Processing Systems}, pp.\
  9015--9025, 2019.

\bibitem[Allen-Zhu et~al.(2019{\natexlab{a}})Allen-Zhu, Li, and
  Liang]{allen2019learning}
Allen-Zhu, Z., Li, Y., and Liang, Y.
\newblock Learning and generalization in overparameterized neural networks,
  going beyond two layers.
\newblock In \emph{Advances in neural information processing systems}, pp.\
  6155--6166, 2019{\natexlab{a}}.

\bibitem[Allen-Zhu et~al.(2019{\natexlab{b}})Allen-Zhu, Li, and
  Song]{allen2019convergence}
Allen-Zhu, Z., Li, Y., and Song, Z.
\newblock A convergence theory for deep learning via over-parameterization.
\newblock In \emph{International Conference on Machine Learning}, pp.\
  242--252, 2019{\natexlab{b}}.

\bibitem[Arora et~al.(2019{\natexlab{a}})Arora, Du, Hu, Li, and
  Wang]{arora2019fine}
Arora, S., Du, S., Hu, W., Li, Z., and Wang, R.
\newblock Fine-grained analysis of optimization and generalization for
  overparameterized two-layer neural networks.
\newblock In \emph{International Conference on Machine Learning}, pp.\
  322--332, 2019{\natexlab{a}}.

\bibitem[Arora et~al.(2019{\natexlab{b}})Arora, Du, Hu, Li, Salakhutdinov, and
  Wang]{arora2019exact}
Arora, S., Du, S.~S., Hu, W., Li, Z., Salakhutdinov, R.~R., and Wang, R.
\newblock On exact computation with an infinitely wide neural net.
\newblock In \emph{Advances in Neural Information Processing Systems}, pp.\
  8139--8148, 2019{\natexlab{b}}.

\bibitem[Arora et~al.(2019{\natexlab{c}})Arora, Du, Li, Salakhutdinov, Wang,
  and Yu]{arora2019harnessing}
Arora, S., Du, S.~S., Li, Z., Salakhutdinov, R., Wang, R., and Yu, D.
\newblock Harnessing the power of infinitely wide deep nets on small-data
  tasks.
\newblock \emph{arXiv preprint arXiv:1910.01663}, 2019{\natexlab{c}}.

\bibitem[Bai \& Lee(2020)Bai and Lee]{bai2020beyond}
Bai, Y. and Lee, J.~D.
\newblock Beyond linearization: On quadratic and higher-order approximation of
  wide neural networks.
\newblock In \emph{International Conference on Learning Representations}, 2020.
\newblock URL \url{https://openreview.net/forum?id=rkllGyBFPH}.

\bibitem[Bradbury et~al.(2018)Bradbury, Frostig, Hawkins, Johnson, Leary,
  Maclaurin, and Wanderman-Milne]{bradbury2018jax}
Bradbury, J., Frostig, R., Hawkins, P., Johnson, M.~J., Leary, C., Maclaurin,
  D., and Wanderman-Milne, S.
\newblock {JAX}: composable transformations of {P}ython+{N}um{P}y programs,
  2018.
\newblock URL \url{http://github.com/google/jax}.

\bibitem[Cao \& Gu(2019)Cao and Gu]{cao2019generalization}
Cao, Y. and Gu, Q.
\newblock Generalization bounds of stochastic gradient descent for wide and
  deep neural networks.
\newblock In \emph{Advances in Neural Information Processing Systems}, pp.\
  10835--10845, 2019.

\bibitem[Chizat \& Bach(2018)Chizat and Bach]{chizat2018global}
Chizat, L. and Bach, F.
\newblock On the global convergence of gradient descent for over-parameterized
  models using optimal transport.
\newblock In \emph{Advances in neural information processing systems}, pp.\
  3036--3046, 2018.

\bibitem[Chizat et~al.(2019)Chizat, Oyallon, and Bach]{chizat2019lazy}
Chizat, L., Oyallon, E., and Bach, F.
\newblock On lazy training in differentiable programming.
\newblock 2019.

\bibitem[Daniely(2017)]{daniely2017sgd}
Daniely, A.
\newblock Sgd learns the conjugate kernel class of the network.
\newblock In \emph{Advances in Neural Information Processing Systems}, pp.\
  2422--2430, 2017.

\bibitem[Du et~al.(2019)Du, Lee, Li, Wang, and Zhai]{du2019gradient}
Du, S., Lee, J., Li, H., Wang, L., and Zhai, X.
\newblock Gradient descent finds global minima of deep neural networks.
\newblock In \emph{International Conference on Machine Learning}, pp.\
  1675--1685, 2019.

\bibitem[Du et~al.(2018)Du, Zhai, Poczos, and Singh]{du2018gradient}
Du, S.~S., Zhai, X., Poczos, B., and Singh, A.
\newblock Gradient descent provably optimizes over-parameterized neural
  networks.
\newblock \emph{arXiv preprint arXiv:1810.02054}, 2018.

\bibitem[Ghorbani et~al.(2019{\natexlab{a}})Ghorbani, Mei, Misiakiewicz, and
  Montanari]{ghorbani2019limitations}
Ghorbani, B., Mei, S., Misiakiewicz, T., and Montanari, A.
\newblock Limitations of lazy training of two-layers neural network.
\newblock In \emph{Advances in Neural Information Processing Systems}, pp.\
  9108--9118, 2019{\natexlab{a}}.

\bibitem[Ghorbani et~al.(2019{\natexlab{b}})Ghorbani, Mei, Misiakiewicz, and
  Montanari]{ghorbani2019linearized}
Ghorbani, B., Mei, S., Misiakiewicz, T., and Montanari, A.
\newblock Linearized two-layers neural networks in high dimension.
\newblock \emph{arXiv preprint arXiv:1904.12191}, 2019{\natexlab{b}}.

\bibitem[Glorot \& Bengio(2010)Glorot and Bengio]{glorot2010understanding}
Glorot, X. and Bengio, Y.
\newblock Understanding the difficulty of training deep feedforward neural
  networks.
\newblock In \emph{Proceedings of the thirteenth international conference on
  artificial intelligence and statistics}, pp.\  249--256, 2010.

\bibitem[He et~al.(2015)He, Zhang, Ren, and Sun]{he2015delving}
He, K., Zhang, X., Ren, S., and Sun, J.
\newblock Delving deep into rectifiers: Surpassing human-level performance on
  imagenet classification.
\newblock In \emph{Proceedings of the IEEE international conference on computer
  vision}, pp.\  1026--1034, 2015.

\bibitem[Hu et~al.(2020)Hu, Li, and Yu]{hu2020simple}
Hu, W., Li, Z., and Yu, D.
\newblock Simple and effective regularization methods for training on noisily
  labeled data with generalization guarantee.
\newblock In \emph{International Conference on Learning Representations}, 2020.
\newblock URL \url{https://openreview.net/forum?id=Hke3gyHYwH}.

\bibitem[Jacot et~al.(2018)Jacot, Gabriel, and Hongler]{jacot2018neural}
Jacot, A., Gabriel, F., and Hongler, C.
\newblock Neural tangent kernel: Convergence and generalization in neural
  networks.
\newblock In \emph{Advances in neural information processing systems}, pp.\
  8571--8580, 2018.

\bibitem[LeCun et~al.(2015)LeCun, Bengio, and Hinton]{lecun2015deep}
LeCun, Y., Bengio, Y., and Hinton, G.
\newblock Deep learning.
\newblock \emph{nature}, 521\penalty0 (7553):\penalty0 436--444, 2015.

\bibitem[Lee et~al.(2018)Lee, Sohl-dickstein, Pennington, Novak, Schoenholz,
  and Bahri]{lee2018deep}
Lee, J., Sohl-dickstein, J., Pennington, J., Novak, R., Schoenholz, S., and
  Bahri, Y.
\newblock Deep neural networks as gaussian processes.
\newblock In \emph{International Conference on Learning Representations}, 2018.
\newblock URL \url{https://openreview.net/forum?id=B1EA-M-0Z}.

\bibitem[Lee et~al.(2019)Lee, Xiao, Schoenholz, Bahri, Novak, Sohl-Dickstein,
  and Pennington]{lee2019wide}
Lee, J., Xiao, L., Schoenholz, S., Bahri, Y., Novak, R., Sohl-Dickstein, J.,
  and Pennington, J.
\newblock Wide neural networks of any depth evolve as linear models under
  gradient descent.
\newblock In \emph{Advances in neural information processing systems}, pp.\
  8570--8581, 2019.

\bibitem[Li \& Liang(2018)Li and Liang]{li2018learning}
Li, Y. and Liang, Y.
\newblock Learning overparameterized neural networks via stochastic gradient
  descent on structured data.
\newblock In \emph{Advances in Neural Information Processing Systems}, pp.\
  8157--8166, 2018.

\bibitem[Li et~al.(2019{\natexlab{a}})Li, Wei, and Ma]{li2019towards}
Li, Y., Wei, C., and Ma, T.
\newblock Towards explaining the regularization effect of initial large
  learning rate in training neural networks.
\newblock In \emph{Advances in Neural Information Processing Systems}, pp.\
  11669--11680, 2019{\natexlab{a}}.

\bibitem[Li et~al.(2019{\natexlab{b}})Li, Wang, Yu, Du, Hu, Salakhutdinov, and
  Arora]{li2019enhanced}
Li, Z., Wang, R., Yu, D., Du, S.~S., Hu, W., Salakhutdinov, R., and Arora, S.
\newblock Enhanced convolutional neural tangent kernels.
\newblock \emph{arXiv preprint arXiv:1911.00809}, 2019{\natexlab{b}}.

\bibitem[Matthews et~al.(2018)Matthews, Rowland, Hron, Turner, and
  Ghahramani]{matthews2018gaussian}
Matthews, A. G. d.~G., Rowland, M., Hron, J., Turner, R.~E., and Ghahramani, Z.
\newblock Gaussian process behaviour in wide deep neural networks.
\newblock \emph{arXiv preprint arXiv:1804.11271}, 2018.

\bibitem[Mei et~al.(2018)Mei, Montanari, and Nguyen]{mei2018mean}
Mei, S., Montanari, A., and Nguyen, P.
\newblock A mean field view of the landscape of two-layers neural networks.
\newblock \emph{Proceedings of the National Academy of Sciences}, 115:\penalty0
  E7665--E7671, 2018.

\bibitem[Mondelli \& Montanari(2019)Mondelli and
  Montanari]{mondelli2019connection}
Mondelli, M. and Montanari, A.
\newblock On the connection between learning two-layer neural networks and
  tensor decomposition.
\newblock In \emph{The 22nd International Conference on Artificial Intelligence
  and Statistics}, pp.\  1051--1060, 2019.

\bibitem[Mu et~al.(2020)Mu, Liang, and Li]{mu2020gradients}
Mu, F., Liang, Y., and Li, Y.
\newblock Gradients as features for deep representation learning.
\newblock In \emph{International Conference on Learning Representations}, 2020.
\newblock URL \url{https://openreview.net/forum?id=BkeoaeHKDS}.

\bibitem[Neal(1996)]{neal1996priors}
Neal, R.~M.
\newblock Priors for infinite networks.
\newblock In \emph{Bayesian Learning for Neural Networks}, pp.\  29--53.
  Springer, 1996.

\bibitem[Novak et~al.(2020)Novak, Xiao, Hron, Lee, Alemi, Sohl-Dickstein, and
  Schoenholz]{novak2020neural}
Novak, R., Xiao, L., Hron, J., Lee, J., Alemi, A.~A., Sohl-Dickstein, J., and
  Schoenholz, S.~S.
\newblock Neural tangents: Fast and easy infinite neural networks in python.
\newblock In \emph{International Conference on Learning Representations}, 2020.
\newblock URL \url{https://openreview.net/forum?id=SklD9yrFPS}.

\bibitem[Pearlmutter(1994)]{pearlmutter1994fast}
Pearlmutter, B.~A.
\newblock Fast exact multiplication by the hessian.
\newblock \emph{Neural computation}, 6\penalty0 (1):\penalty0 147--160, 1994.

\bibitem[Rotskoff \& Vanden-Eijnden(2018)Rotskoff and
  Vanden-Eijnden]{rotskoff2018neural}
Rotskoff, G.~M. and Vanden-Eijnden, E.
\newblock Neural networks as interacting particle systems: Asymptotic convexity
  of the loss landscape and universal scaling of the approximation error.
\newblock \emph{arXiv preprint arXiv:1805.00915}, 2018.

\bibitem[Sirignano \& Spiliopoulos(2018)Sirignano and
  Spiliopoulos]{sirignano2018mean}
Sirignano, J. and Spiliopoulos, K.
\newblock Mean field analysis of neural networks.
\newblock \emph{arXiv preprint arXiv:1805.01053}, 2018.

\bibitem[Wei et~al.(2019)Wei, Lee, Liu, and Ma]{wei2019regularization}
Wei, C., Lee, J.~D., Liu, Q., and Ma, T.
\newblock Regularization matters: Generalization and optimization of neural
  nets vs their induced kernel.
\newblock In \emph{Advances in Neural Information Processing Systems}, pp.\
  9709--9721, 2019.

\bibitem[Yang(2019)]{yang2019tensor}
Yang, G.
\newblock Tensor programs i: Wide feedforward or recurrent neural networks of
  any architecture are gaussian processes.
\newblock \emph{arXiv preprint arXiv:1910.12478}, 2019.

\bibitem[Yehudai \& Shamir(2019)Yehudai and Shamir]{yehudai2019power}
Yehudai, G. and Shamir, O.
\newblock On the power and limitations of random features for understanding
  neural networks.
\newblock In \emph{Advances in Neural Information Processing Systems}, pp.\
  6594--6604, 2019.

\bibitem[Zagoruyko \& Komodakis(2016)Zagoruyko and
  Komodakis]{zagoruyko2016wide}
Zagoruyko, S. and Komodakis, N.
\newblock Wide residual networks.
\newblock \emph{arXiv preprint arXiv:1605.07146}, 2016.

\bibitem[Zhang et~al.(2019)Zhang, Bengio, and Singer]{zhang2019all}
Zhang, C., Bengio, S., and Singer, Y.
\newblock Are all layers created equal?
\newblock \emph{arXiv preprint arXiv:1902.01996}, 2019.

\bibitem[Zou et~al.(2019)Zou, Cao, Zhou, and Gu]{zou2019gradient}
Zou, D., Cao, Y., Zhou, D., and Gu, Q.
\newblock Gradient descent optimizes over-parameterized deep relu networks.
\newblock \emph{Machine Learning}, pp.\  1--26, 2019.

\end{thebibliography}
\bibliographystyle{icml2020}

\appendix
\onecolumn
\section{Proof of Theorem~\ref{theorem:main}}
\label{appendix:proof-main}

% \yub{Current proof written in standard param. Need to change to NTK
%   param. Lemma 3 requires more care.}
\subsection{Formal statement of Theorem~\ref{theorem:main}}
We first collect notation and state our assumptions. Recall that our
two-layer neural network is defined as
\begin{equation*}
  f_W(x) = \frac{1}{\sqrt{m}}\sum_{r=1}^m a_r\sigma(w_r^\top x),
\end{equation*}
and its $k$-th order Taylorized model is
\begin{equation}
  \label{equation:two-layer-taylorized}
  f^{(k)}_W(x) = \frac{1}{\sqrt{m}}\sum_{r=1}^m a_r \sum_{j=0}^k
  \frac{\sigma^{(j)}(w_{r,0}^\top x)}{j!} \brac{(w_r - w_{r,0})^\top x}^j.
\end{equation}
Let
$\mc{X}\in\R^{d\times n}$ and $\mc{Y}\in\R^n$ denote inputs and labels
of the training dataset. For any weight matrix $W\in\R^{d\times m}$ we
let 
\begin{align*}
  & f(W) = f(W;\mc{X}) \defeq \begin{bmatrix}
    f_W(x_1) \\
    \vdots \\
    f_W(x_n)
  \end{bmatrix} \in \R^n, ~~~
  & f^{(k)}(W) = f^{(k)}(W; \mc{X}) \defeq \begin{bmatrix}
    f^{(k)}_W(x_1) \\
    \vdots \\
    f^{(k)}_W(x_n)
  \end{bmatrix} \in \R^n, \\
  & g(W) = f(W) - \mc{Y} \in \R^n, ~~~
  & g^{(k)}(W) = f^{(k)}(W) - \mc{Y} \in \R^n,
  \\
  & f_t \defeq f(W_t),~~~&f^{(k)}_t \defeq f^{(k)}(W^{(k)}_t), \\
  & g_t \defeq g(W_t),~~~&g^{(k)}_t \defeq g^{(k)}(W^{(k)}_t), \\
  & J(W) \defeq \begin{bmatrix}
    \grad f_W(x_1) \\
    \vdots \\
    \grad f_W(x_n)
  \end{bmatrix} \in \R^{n\times dm},~~~
  & J^{(k)}(W) \defeq \begin{bmatrix}
    \grad f^{(k)}_W(x_1) \\
    \vdots \\
    \grad f^{(k)}_W(x_n)
  \end{bmatrix} \in \R^{n\times dm}.
\end{align*}
With this notation, the loss functions can be written as
$L(W)=\norm{g(W)}^2$ and $L^{(k)}(W)=\norm{g^{(k)}(W)}^2$, and the
training dynamics (full and Taylorized) can be written as
\begin{align*}
  & \dot{W}_t = -\eta_0\cdot J(W_t)^\top g_t,\qquad \dot{g}_t =
    -\eta_0 \cdot J(W_t)J(W_t)^\top g_t, \\
  & \dot{W}^{(k)}_t = -\eta_0\cdot J^{(k)}(W^{(k)}_t)^\top
    g^{(k)}_t,\qquad \dot{g}^{(k)}_t = 
    -\eta_0 \cdot J^{(k)}(W^{(k)}_t)J^{(k)}(W^{(k)}_t)^\top
    g^{(k)}_t. 
\end{align*}

We now state our assumptions.
\begin{assumption}[Full-rankness of analytic NTK]
  \label{assumption:full-rank-ntk}
  The analytic NTK $\Theta\in\R^{n\times n}$ on the training dataset,
  defined as
  \begin{equation*}
    \Theta \defeq \lim_{m\to\infty} J(W_0)J(W_0)^\top~~~\textrm{in
      probability},
  \end{equation*}
  is full rank and satisfies $\lambda_{\min}(\Theta)\ge \lambda_{\min}$
  for some $\lambda_{\min}>0$.
\end{assumption}

\begin{assumption}[Smooth activation]
  \label{assumption:smooth-activation}
  The activation function $\sigma:\R\to\R$ is $C^k$ and has a bounded
  Lipschitz derivative: there exists a constant $C>0$ such that
  \begin{align*}
    \sup_{t\in\R} \abs{\sigma'(t)}\le C~~~{\rm and}~~~\sup_{t\neq
    t'\in\R} \abs{\frac{\sigma'(t) - \sigma'(t')}{t-t'}} \le C.
  \end{align*}
  Further, $\sigma'$ has a Lipschitz $k$-th derivative:
  $t\mapsto\sigma^{(k)}(t)$ is $L$-Lipschitz for some constant $L>0$.
\end{assumption}
% Note that this assumption implies that $\sigma$, along with its
% $\set{1,\dots,k}$-th order derivative, are uniformly bounded on any
% bounded interval $[-M, M]$. (This can be shown by Taylor expanding
% $\sigma$ around $0$ to the $k$-th order and applying the $k$-th order
% smoothness.)

%\yub{more assumptions.}

Throughout the rest of this section, we assume the above assumptions
hold. We are now in position to formally state our main theorem.
\begin{theorem}[Approximation error of Taylorized training; formal
  version of Theorem~\ref{theorem:main}]
  \label{theorem:main-formal}
  There exists a suitable step-size choice $\eta_0>0$ such that the
  following is true: for any fixed $t_0>0$ and all sufficiently large
  $m$,
  with high probability over the random initialization, % the trajectories of
  full training~\eqref{equation:full-training-gf} and 
  Taylorized training~\eqref{equation:taylorized-training-gf} are coupled
  in both the parameter space and the function space:
  \begin{align*}
    & \sup_{t\le t_0} \lfro{W_t - W^{(k)}_t} \le O(m^{-k/2}), \\
    & \sup_{t\le t_0} \abs{f_{W_t}(x) - f^{(k)}_{W^{(k)}_t}(x)} \le
      O(m^{-k/2})~~\textrm{for all}~\ltwo{x}=1.
  \end{align*}
\end{theorem}
\paragraph{Remark on extending to entire trajectory} Compared with the
existing result on linearized training~\citep[][Theorem
H.1]{lee2019wide}, our Theorem~\ref{theorem:main-formal} only
shows the approximation for a fixed time horizon $t\in[0, t_0]$
instead of the entire trajectory $t\in[0, \infty)$. Technically, this
is due to that the linearized result uses a more careful Gronwall type
argument which relies on the fact that the kernel does not change,
which ceases to hold here. It would be a compelling question if we
could show the approximation result for higher-order Taylorized
training for the entire trajectory.

\subsection{Proof of Theorem~\ref{theorem:main-formal}}
Throughout the proof, we let $K$ be a constant that does not depend
on $m$, but can depend on other problem parameters and can vary from
line to line. We will also denote
\begin{equation}
  \label{equation:sigma-k}
  \sigma^{(k)}_{r,i}(t) \defeq \sum_{j=0}^k
  \frac{\sigma^{(j)}(w_{r,0}^\top x_i)}{j!} \brac{t - w_{r,0}^\top
  x_i}^j,
\end{equation}
so that the Taylorized model can be essentially thought of as a
two-layer neural network with the (data- and neuron-dependent)
activation functions $\sigma^{(k)}_{r,i}$.

We first present some known results about the full training trajectory
$W_t$, adapted from~\citep[][Appendix
G]{lee2019wide}.
\begin{lemma}[Basic properties of full training]
  \label{lemma:basic-full-training}
  Under
  Assumptions~\ref{assumption:full-rank-ntk},~\ref{assumption:smooth-activation},
  the followings hold:
  \begin{enumerate}[(a)]
  \item $J$ is locally bounded and Lipschitz:
    For any absolute constant $C>0$ there exists a constant $K>0$ such
    that for sufficiently large $m$, with high probability (over the
    random initialization $W_0$) we have
    \begin{align*}
      & \lfro{J(W) - J(\wt{W})} \le K\lfro{W -
        \wt{W}}~~~{\rm and} \\
      & \lfro{J(W)} \le K
    \end{align*}
    for any $W,\wt{W}\in \ball(W_0, C)$, where $\ball(W_0, R)\defeq
    \set{W\in\R^{d\times m}:\lfro{W-W_0}\le R}$ denotes a Frobenius norm
    ball.
  \item Boundedness of gradient flow: there exists an absolute $R_0>0$
    such that with high probability for sufficiently large $m$ and a
    suitable step-size choice $\eta_0$ (independent of $m$), we have
    for all $t>0$ that
    \begin{align*}
      & \ltwo{g_t} \le \exp(-\eta_0\lambda_{\min}t/3)R_0, \\
      & \lfro{W_t - W_0} \le \frac{KR_0}{\lambda_{\min}}\paren{1 -
        \exp(-\eta_0\lambda_{\min}t/3)}m^{-1/2}, \\
      & \sup_{t\ge 0} \lfro{\what{\Theta}_0 - \what{\Theta}_t} \le
        \frac{K^3R_0}{\lambda_{\min}}m^{-1/2}.
    \end{align*}
  \end{enumerate}
\end{lemma}

\begin{lemma}[Properties of Taylorized training]
  \label{lemma:basic-taylorized-training}
  Lemma~\ref{lemma:basic-full-training} also holds if we replace full
  training with $k$-th order Taylorized training. More concretely, we
  have
    \begin{enumerate}[(a)]
  \item $J^{(k)}$ is locally bounded and Lipschitz:
    For any absolute constant $C>0$ there exists a constant $K>0$ such
    that for sufficiently large $m$, with high probability (over the
    random initialization $W_0$) we have
    \begin{align*}
      & \lfro{J^{(k)}(W) - J^{(k)}(\wt{W})} \le K\lfro{W -
        \wt{W}}~~~{\rm and} \\
      & \lfro{J^{(k)}(W)} \le K
    \end{align*}
    for any $W,\wt{W}\in \ball(W_0, C)$, where $\ball(W_0, R)\defeq
    \set{W\in\R^{d\times m}:\lfro{W-W_0}\le R}$ denotes a Frobenius norm
    ball.
  \item Boundedness of gradient flow: there exists an absolute $R_0>0$
    such that with high probability for sufficiently large $m$ and a
    suitable step-size choice $\eta_0$ (independent of $m$), we have
    for all $t>0$ that
    \begin{align*}
      & \ltwo{g^{(k)}_t} \le \exp(-\eta_0\lambda_{\min}t/3)R_0, \\
      & \lfro{W^{(k)}_t - W_0} \le \frac{KR_0}{\lambda_{\min}}\paren{1 -
        \exp(-\eta_0\lambda_{\min}t/3)}m^{-1/2}, \\
      & \sup_{t\ge 0} \lfro{\what{\Theta}_0 - \what{\Theta}^{(k)}_t} \le
        \frac{K^3R_0}{\lambda_{\min}}m^{-1/2}.
    \end{align*}
  \end{enumerate}
\end{lemma}
\begin{proof}
  \begin{enumerate}[(a)]
  \item Rewrite the $k$-th order Taylorized
    model~\eqref{equation:two-layer-taylorized} as
    \begin{equation}
      f^{(k)}_W(x_i) = \frac{1}{\sqrt{m}}\sum_{r=1}^m a_r
      \sigma_{r,i}^{(k)}(w_r^\top x_i),
    \end{equation}
    where we have used the definition of the ``Taylorized'' activation
    function $\sigma^{(k)}_{r,i}$ in~\eqref{equation:sigma-k}.

    Our goal here is to show that $J^{(k)}$ is $K$-bounded and
    $K$-Lipschitz for some absolute constant $K$. By
    Lemma~\ref{lemma:basic-full-training}, it suffices to show the
    same for $J-J^{(k)}$, as we already have the result for the
    original Jacobian $J$. Let $W\in\ball(W_0,C)$, we have
    \begin{equation}
      \label{equation:jacobian-expansion}
      \begin{aligned}
        & \quad \lfro{J(W) - J^{(k)}(W)}^2 = \sum_{i=1}^n \frac{1}{m}
        \sum_{r=1}^m \paren{ \sigma'(w_r^\top x_i) -
          [\sigma^{(k)}]'(w_r^\top x_i) }^2 \\
        & \stackrel{(i)}{\le} \sum_{i=1}^n \frac{1}{m} \sum_{r=1}^m
        \paren{\frac{L}{k!}}^2 \paren{w_r^\top x_i - w_{r,0}^\top
          x_i}^{2k} \\
        & \le \frac{K}{m} \sum_{r=1}^m \ltwo{w_r - w_{r,0}}^{2k} \le
        \frac{K}{m} \paren{\sum_{r=1}^m \ltwo{w_r - w_{r,0}}^2}^k 
        \le KC^k/m \le K.
      \end{aligned}
    \end{equation}
    Above, (i) uses the $k$-th order smoothness of $\sigma$. This
    shows the boundedness of $J - J^{(k)}$.

    A similar
    argument can be done for the Lipschitzness of $J - J^{(k)}$, where the
    second-to-last expression is replaced by $\sum_{r\le m}\ltwo{w_r -
      w_{r,0}}^{2(k-1)}$, from which the same argument goes through as
    $2(k-1)\ge 2$ whenever $k\ge 2$, and for $k=1$ the sum is bounded
    by $K/m\cdot m=K$.

  \item This is a direct corollary of part (a), as we can view the
    Taylorized network as an architecture on its own, which has the
    same NTK as $f$ at init (so the non-degeneracy of the NTK also
    holds), and has a locally bounded Lipschitz Jacobian. Repeating
    the argument of~\citep[][Theorem G.2]{lee2019wide} gives the
    results. 
  \end{enumerate}
  % \yub{Below deprecated.}
  
  % The result for the original Jacobian $J$ is a special case of
  % \citep[Lemma 4,][Appendix G]{lee2019wide}. We now show the same result 
  % for the Taylorized Jacobian $J^{(k)}$. First, given any $W\in \ball(W_0,
  % C/\sqrt{m})$, we have $\ltwo{w_r - w_{r,0}}\le C/\sqrt{m}$, and thus
  % for all $r\in[m]$ we have
  % \begin{equation*}
  %   \ltwo{w_r} \le \ltwo{w_{r,0}} + \ltwo{w_r - w_{r,0}} \le
  %   K\sqrt{\frac{d\log m}{m}} + \frac{C}{\sqrt{m}} \le 1
  % \end{equation*}
  % for sufficiently large $m$ (where the bound on $\ltwo{w_{r,0}}$
  % follows from standard Gaussian norm
  % concentration~\cite{vershynin2018high}. ) This further implies for
  % the considered $W,\wt{W}$ that 
  % \begin{equation*}
  %   \sup_{\ltwo{x}=1,r\in[m]} \max\set{ |w_{r,0}^\top x|, |w_r^\top
  %     x|, |\wt{w}_r^\top x|} \le 1.
  % \end{equation*}
  % Therefore we have
  % \begin{align*}
  %   \frac{1}{m}\lfro{J^{(k)}(W)}^2 =
  %   \frac{1}{m}\sum_{i=1}^n\sum_{r=1}^m (\sigma^{(k)}_{r,i})'(w_r^\top x_i)^2
  %   \ltwo{x_i}^2 \le K.
  % \end{align*}
  % Similarly we have
  % \begin{align*}
  %   & \quad \frac{1}{m}\lfro{J^{(k)}(W) - J^{(k)}(\wt{W})}^2 \\
  %   & = \frac{1}{m}
  %     \sum_{i=1}^n \sum_{r=1}^m \brac{(\sigma^{(k)}_{r,i})'(w_r^\top
  %     x_i) - (\sigma^{(k)}_{r,i})'(\wt{w}_r^\top x_i)}^2 \ltwo{x_i}^2
  %     \le \frac{1}{m}\sum_{i=1}^n\sum_{r=1}^m K\ltwo{w_r - \wt{w}_r}^2
  %     \le \frac{K}{m}\lfro{W - \wt{W}}^2.
  % \end{align*}
  % (Note that the Lipschitz constant is better than the Lemma
  % statement. This is due to that we did not consider the Jacobian with
  % respect to the top-layer $a$.)
\end{proof}

% \begin{corollary}
%   \label{corollary:weight-in-ball}
%   For sufficiently large $m$, we have with high probability that
%   $\ltwo{g(W_0)}\le R_0$ for some constant $R_0>0$. Further, with a
%   proper step-size choice $\eta_0$, we have
%   $\ltwo{g_t}\le \exp(-\eta_0\lambda_{\min}t/3)R_0$ for all $t>0$ and
%   \begin{equation*}
%     \max\set{\sup_{t>0} \lfro{W_t - W_0}, \sup_{t>0}\lfro{W^{(k)}_t - W_0}}
%     \le \frac{KR_0}{\lambda_{\min}} (1 - 
%     e^{-\eta_0\lambda_{\min}t/3})m^{-1/2}.
%   \end{equation*}
%   In other words, both the full training and Taylorized training
%   does not leave a $O(1/\sqrt{m})$ Frobenius norm ball around the
%   initialization. 
% \end{corollary}
% \begin{proof}
%   The bounds for $\ltwo{g(W_0)}$ and $\sup_t\lfro{W_t-W_0}$ are shown
%   in Theorem G.2 of~\cite{lee2019wide}, and the bound for
%   $\sup_t\lfro{W^{(k)}_t-W_0}$ follows similarly as the Taylorized
%   Jacobian $J^{(k)}$ satisfies the same local Lipschitzness bound
%   (Lemma~\ref{lemma:jacobian-local-lipschitz}), and the fact that the
%   NTK for $f^{(k)}$ coincides with that of $f$ for all $k\ge 1$.
% \end{proof}

\begin{lemma}[Bounding invididual weight movements in $W_t$ and
  $W^{(k)}_t$]
  \label{lemma:individual-weight-movement}
  Under the same settings as Lemma~\ref{lemma:basic-full-training}
  and~\ref{lemma:basic-taylorized-training}, we have
  \begin{equation}
    \left\{
      \begin{aligned}
        & \max_{r\in[m]} \ltwo{w_{r,t} - w_{r,0}} \le \frac{KR_0}{\lambda_{\min}} (1 - 
        e^{-\eta_0\lambda_{\min}t/3})m^{-1/2}, \\
        & \max_{r\in[m]} \ltwo{w^{(k)}_{r,t} - w_{r,0}} \le \frac{KR_0}{\lambda_{\min}} (1 - 
        e^{-\eta_0\lambda_{\min}t/3})m^{-1/2}.
      \end{aligned}
    \right.
  \end{equation}
  Consequently, we have for $\wt{W}_t\in\set{W_t, W^{(k)}_t}$ that
  \begin{align*}
    \lfro{J(\wt{W}_t) - J^{(k)}(\wt{W}_t)} \le
    \paren{\frac{KR_0}{\lambda_{\min}}}^k \paren{1 -
    \exp(-\eta_0\lambda_{\min}t/3)}^km^{-k/2}.
  \end{align*}
\end{lemma}
\begin{proof}
  We first show the bound for $\ltwo{w_{r,t} - w_{r,0}}$, and the bound
  for $\ltwo{w^{(k)}_{r,t} - w_{r,0}}$ follows similarly. We have
  \begin{align*}
    \frac{d}{dt}\ltwo{w_{r,t} - w_{r,0}} \le \ltwo{\frac{d}{dt}
    w_{r,t}} \le \eta_0 \ltwo{J_{w_r}(W_t)^\top g_t} \le \eta_0\cdot
    \lfro{J_{w_r}(W_t)} \ltwo{g_t}.
  \end{align*}
  Note that
  \begin{align*}
    \lfro{J_{w_r}(W_t)}^2 = \frac{1}{m}\sum_{i=1}^n
    \sigma'(w_{r,t}^\top x_i)^2 \ltwo{x_i}^2 \le \frac{1}{m} \cdot C^2
    = C^2/m.
  \end{align*}
  due to the boundedness of $\sigma'$, and $\ltwo{g_t}\le
  \exp(-\eta_0\lambda_{\min}t/3)R_0$ by
  Lemma~\ref{lemma:basic-full-training}(b), so we have
  \begin{align*}
    & \quad \frac{d}{dt} \ltwo{w_{r,t} - w_{r,0}} \le
      \eta_0\lfro{J_{w_r}(W_t)}\ltwo{g_t} \\
    & \le \eta_0\cdot K/\sqrt{m}
      \cdot \exp(-\eta_0\lambda_{\min}t/3)R_0 =
      K\eta_0R_0\exp(-\eta_0\lambda_{\min}t/3)m^{-1/2},
  \end{align*}
  integrating which (and noticing the initial condition
  $\ltwo{w_{r,t}-w_{r,0}}|_{t=0}=0$) yields that
  \begin{align*}
    \ltwo{w_{r,t} - w_{r,0}} \le
    \frac{3KR_0}{\lambda_{\min}}\paren{1 -
    \exp(-\eta_0\lambda_{\min}t/3)}m^{-1/2}. 
  \end{align*}

  We now show the bound on $\lfro{J - J^{(k)}}$, again focusing on the
  case $\wt{W}_t\equiv W_t$ (and the case $\wt{W}_t\equiv W^{(k)}_t$
  follows similarly). By~\eqref{equation:jacobian-expansion}, we
  have
  \begin{align*}
    & \quad \lfro{J(W_t) - J^{(k)}(W_t)}^2 \le \frac{K}{m}\sum_{r=1}^m
      \ltwo{w_r - w_{r,0}}^{2k} \le
      \frac{1}{m}\sum_{r=1}^m \paren{\frac{KR_0}{\lambda_{\min}}}^{2k}
      \cdot (1 - \exp(-\eta_0\lambda_{\min}t/3))^{2k} m^{-k} \\
    & = \paren{\frac{KR_0}{\lambda_{\min}}}^{2k}
      \cdot (1 - \exp(-\eta_0\lambda_{\min}t/3))^{2k} m^{-k}.
  \end{align*}
  Taking the square root gives the desired result.
\end{proof}

We are now in position to prove the main
theorem.

\begin{proof-of-theorem}[\ref{theorem:main-formal}]
  {\bf Step 1.}
  We first bound the rate of change of $\lfro{W_t - W^{(k)}_t}$. We
  have
  \begin{align*}
    & \quad \frac{d}{dt} \lfro{W_t - W^{(k)}_t} \le \lfro{\frac{d}{dt} (W_t -
      W^{(k)}_t)} \le \eta_0 \lfro{J(W_t)^\top g_t -
      J^{(k)}(W^{(k)}_t)^\top g^{(k)}_t} \\
    & \le \underbrace{\eta_0\lfro{J(W_t) - J^{(k)}(W_t)}\cdot
      \ltwo{g_t}}_{\rm I} +
      \underbrace{\eta_0\lfro{J^{(k)}(W_t) - J^{(k)}(W^{(k)}_t)}\cdot
      \ltwo{g_t}}_{\rm II} +
      \underbrace{\eta_0\lfro{J^{(k)}(W^{(k)}_t)} \cdot \ltwo{g_t -
      g^{(k)}_t}}_{\rm III}.
  \end{align*}
  For term I, applying Lemma~\ref{lemma:individual-weight-movement}
  yields
  \begin{align*}
    & \quad {\rm I} \le \eta_0 \paren{\frac{KR_0}{\lambda_{\min}}}^{k}
      \cdot (1 - \exp(-\eta_0\lambda_{\min}t/3))^{k} m^{-k/2} \cdot
      \exp(-\eta_0\lambda_{\min}t/3) \cdot R_0 \\
    & \le \eta_0R_0 \paren{\frac{KR_0}{\lambda_{\min}}}^{k}\cdot
      \exp(-\eta_0\lambda_{\min}t/3)m^{-k/2}.
  \end{align*}
  For term II, applying the local Lipschitzness of $J^{(k)}$ and the
  fact that $W_t,W^{(k)}_t$ are in $\ball(W_0,C)$
  (Lemma~\ref{lemma:basic-taylorized-training})  yields that
  \begin{align*}
    {\rm II} \le \eta_0 \cdot K\lfro{W_t - W^{(k)}_t} \cdot
    \exp(-\eta_0\lambda_{\min}t/3)R_0 = \eta_0KR_0
    \exp(-\eta_0\lambda_{\min}t/3)\cdot \lfro{W_t - W^{(k)}_t}. 
  \end{align*}
  For term III, using the loca boundedness of $J^{(k)}$ gives that
  \begin{align*}
    {\rm III} \le \eta_0 K\cdot \ltwo{g_t - g^{(k)}_t}.
  \end{align*}
  Summing the three bounds together gives a ``master'' bound
  \begin{equation}
    \label{equation:rate-bound-w}
    \begin{aligned}
      & \quad \frac{d}{dt} \lfro{W_t - W^{(k)}_t} \\
      & \le \eta_0R_0K_1
      \exp(-\eta_0\lambda_{\min}t/3) \cdot m^{-k/2} +
      \eta_0K_2R_0\exp(-\eta_0\lambda_{\min}t/3) \cdot \lfro{W_t -
        W_t^{(k)}} + \eta_0K_3\ltwo{g_t - g^{(k)}_t}.
    \end{aligned}
  \end{equation}

  {\bf Step 2.} We now observe that the function-space difference
  $g_t - g^{(k)}_t$ obeys the equation
  \begin{align*}
    \frac{d}{dt} \paren{g_t - g^{(k)}_t} = -\eta_0
    \paren{J(W_t)J(W_t)^\top g_t -
    J^{(k)}(W^{(k)}_t)J^{(k)}(W^{(k)}_t)^\top g^{(k)}_t},
  \end{align*}
  which only differs from the equation for $W_t - W^{(k)}_t$ in having
  one more Jacobian multiplication in front. Therefore, using the
  exact same argument as in Step 1 (and using the local boundedness of
  the Jacobian), we obtain the ``master'' bound for
  $\frac{d}{dt}\ltwo{g_t - g^{(k)}_t}:$
  \begin{equation}
    \label{equation:rate-bound-g}
    \begin{aligned}
      & \quad \frac{d}{dt} \ltwo{g_t - g^{(k)}_t} \\
      & \le \eta_0R_0K_4
      \exp(-\eta_0\lambda_{\min}t/3) \cdot m^{-k/2} +
      \eta_0K_5R_0\exp(-\eta_0\lambda_{\min}t/3) \cdot \lfro{W_t -
        W_t^{(k)}} + \eta_0K_6\ltwo{g_t - g^{(k)}_t}.
    \end{aligned}
  \end{equation}

  {\bf Step 3.} Define
  \begin{align*}
    A(t) \defeq \lfro{W_t - W^{(k)}_t} + \ltwo{g_t - g^{(k)}_t}.
  \end{align*}
  Adding the two master bounds~\eqref{equation:rate-bound-w}
  and~\eqref{equation:rate-bound-g} together, we obtain
  \begin{align*}
    A'(t) \le \eta_0K_7A(t) + \eta_0R_0K_8m^{-k/2},
  \end{align*}
  so by a standard Gronwall inequality argument (i.e. considering a
  change of variable  $B(t) =\exp(-\eta_0K_7t)A(t)$) and using the
  initial condition $A(0)=0$, we obtain
  \begin{align*}
    A(t) \le \frac{\eta_0R_0K_8}{K_7} m^{-k/2}\cdot \exp(\eta_0K_7t). 
  \end{align*}
  Therefore, choosing $\eta_0$ so that
  Lemma~\ref{lemma:basic-full-training},~\ref{lemma:basic-taylorized-training}
  and~\ref{lemma:individual-weight-movement} hold, for any fixed
  $t_0>0$, we have
  \begin{align*}
    \max\set{\lfro{W_t - W^{(k)}_t}, \ltwo{g_t - g^{(k)}_t}}  = A(t)
    \le O(m^{-k/2}),
  \end{align*}
  where $O(\cdot)$ hides constants that depend on $(d, n,
  \lambda_{\min})$ and (potentially exponentially on) $t_0$, but not
  $m$. The bound on $W_t - W^{(k)}_t$ thereby gives the first part of
  the desired result.

  {\bf Step 4.} For any other test data point $x\in\R^d$ such that
  $\ltwo{x}=1$, letting $f_t(x)\defeq f_{W_t}(x)$ and $f^{(k)}_t(x)
  \defeq f^{(k)}_{W^{(k)}_t}(x)$, we have the evolution
  \begin{align*}
    \frac{d}{dt}\paren{f_t(x) - f^{(k)}_t(x)} = -\eta_0\paren{J_t(x)
    J(W_t)^\top g(t) - J^{(k)}_t(x)J^{(k)}(W_t)^\top g^{(k)}_t }.
  \end{align*}
  The local boundedness and Lipschitzness of $J_t(x)$ and
  $J^{(k)}_t(x)$ holds (and can be shown) exactly similarly as in
  Lemma~\ref{lemma:basic-full-training}
  and~\ref{lemma:basic-taylorized-training}. Decomposing the RHS and
  using a similar argument as in Step 3 gives that
  \begin{align*}
    \abs{f_t(x) - f^{(k)}_t(x)} \le O(m^{-k/2}).
  \end{align*}
  This is the second part of the desired result.
\end{proof-of-theorem}
\newpage
\section{Additional experimental results}
\label{appendix:additional-exps}

\subsection{Agreement between Taylorized and full training}
\label{appendix:approximation}
We plot results for Taylorized vs. full training on the \{CNNTHICK,
WRNTHIN, WRNTHICK\} (our three other main architectures) in
Figure~\ref{figure:cnnthick-approximation},~\ref{figure:wrnsmall-approximation},
and~\ref{figure:wrnwide-approximation}, complementing the results on
the CNNTHIN architecture in the main paper
(Section~\ref{section:approximation-exp}).

\begin{figure}[h]
  \centering
    \includegraphics[width=0.24\textwidth]{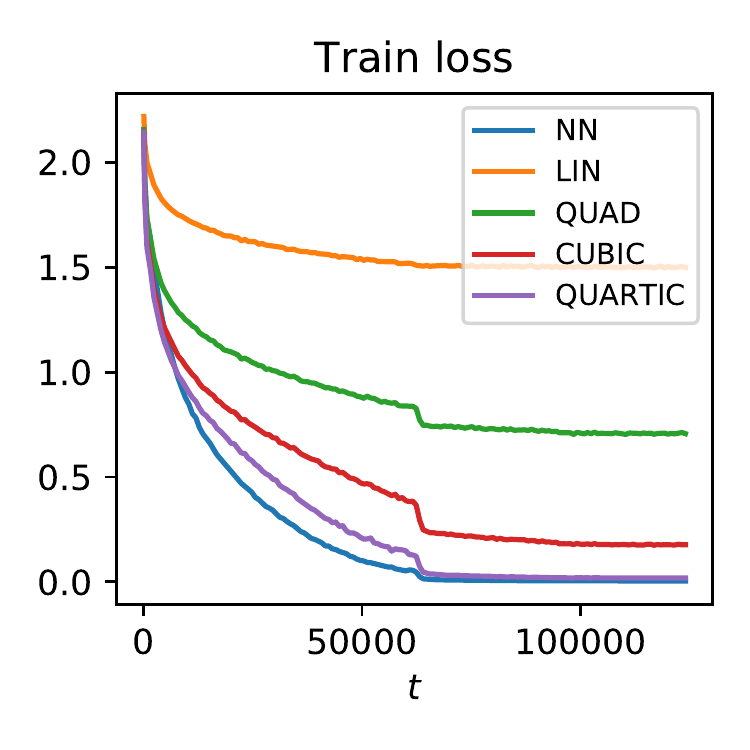}
    \includegraphics[width=0.24\textwidth]{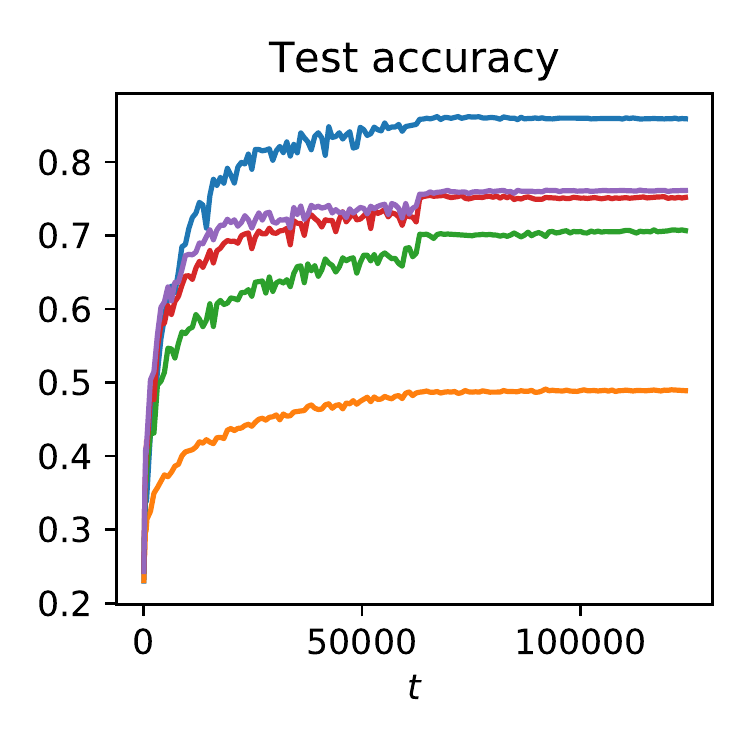}
    \includegraphics[width=0.24\textwidth]{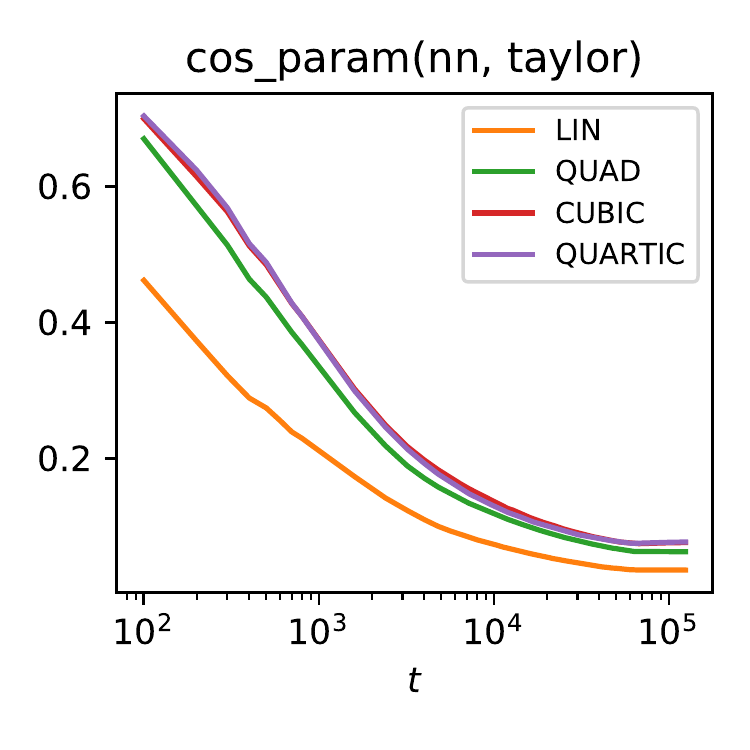}
    \includegraphics[width=0.24\textwidth]{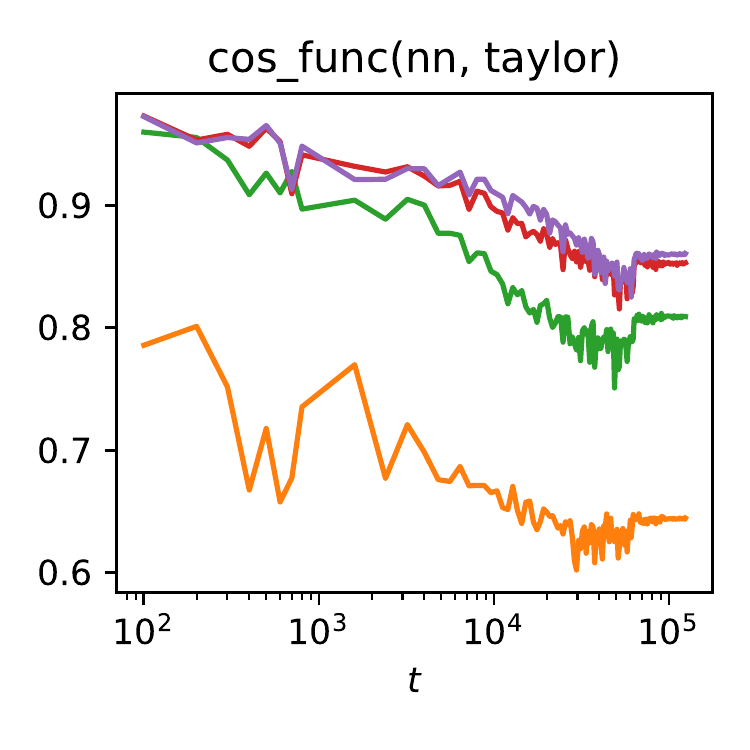}
    \caption{$k$-th order Taylorized training approximates full
      training increasingly better with $k$ on the CNNTHICK
      model. Training statistics are plotted for \{\textbf{{\color{C0}
          full}, {\color{C1} linearized}, {\color{C2} quadratic},
        {\color{C3} cubic}, {\color{C4} quartic}}\} models. Left to
      right: (1) training loss; (2) test accuracy; (3) cosine
      similarity between Taylorized and full training in the parameter
      space; (4) cosine similarity between Taylorized and full
      training in the function (logit) space. All models are trained on
      CIFAR-10 for 124000 steps, and a 10x learning rate decay happened
      at step \{62000, 93000\}.}
  \label{figure:cnnthick-approximation}
\end{figure}

\begin{figure}[h]
  \centering
    \includegraphics[width=0.24\textwidth]{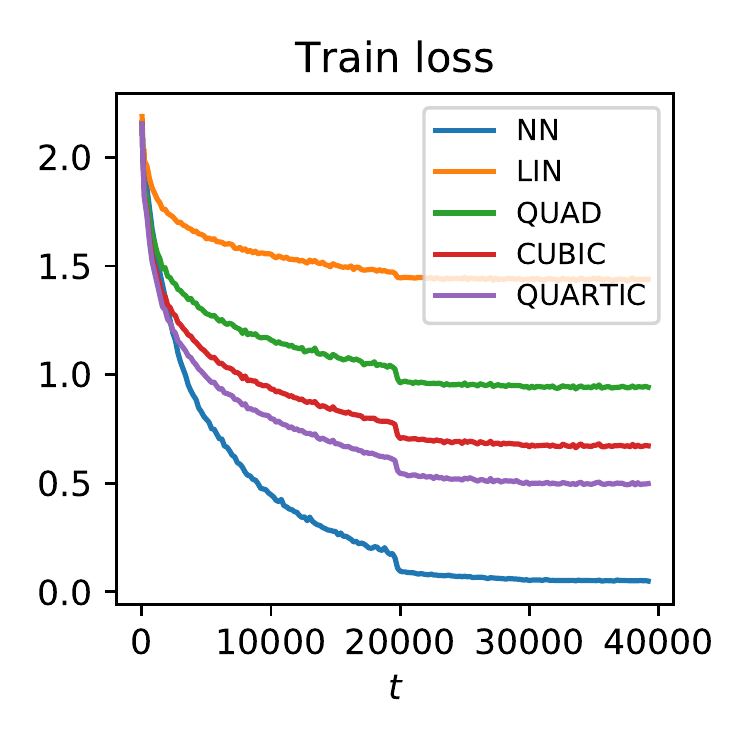}
    \includegraphics[width=0.24\textwidth]{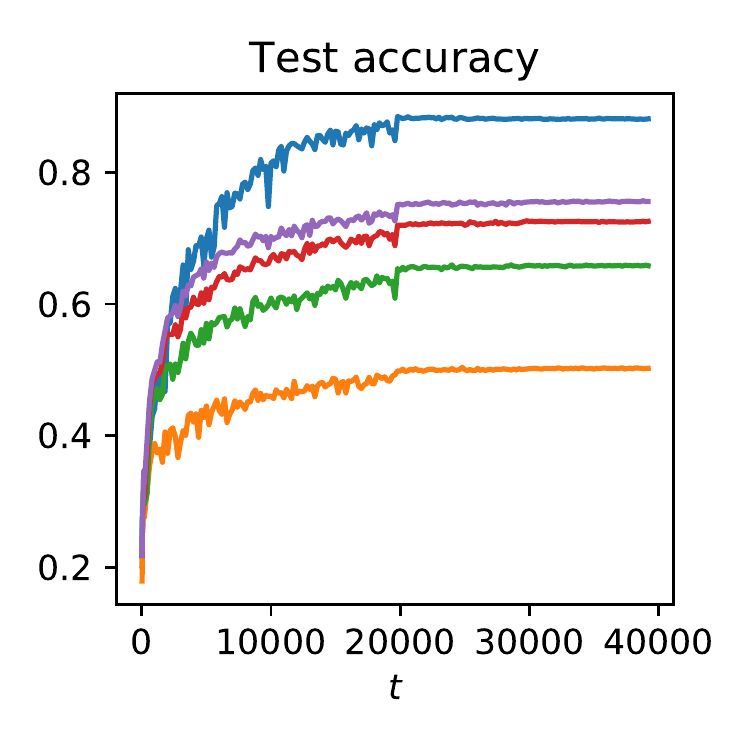}
    \includegraphics[width=0.24\textwidth]{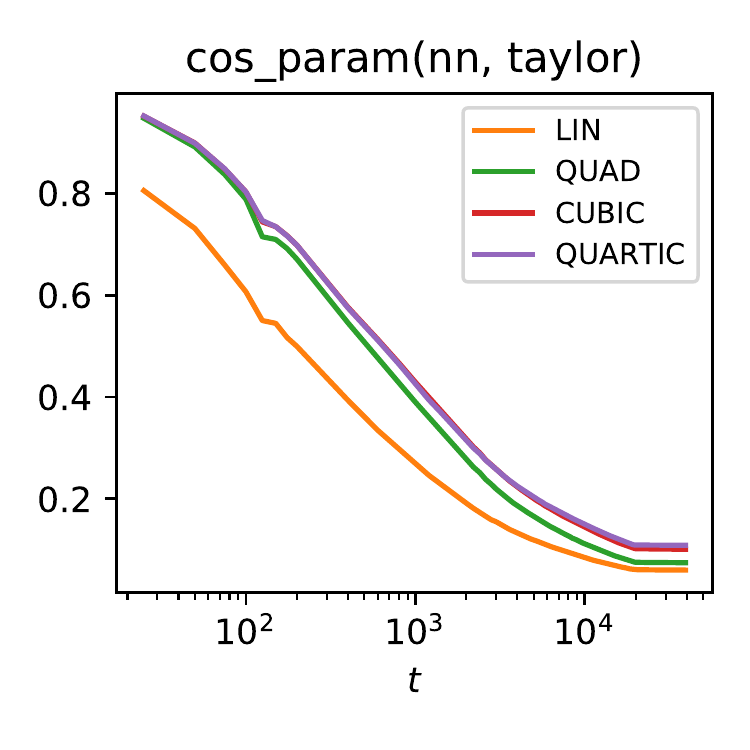}
    \includegraphics[width=0.24\textwidth]{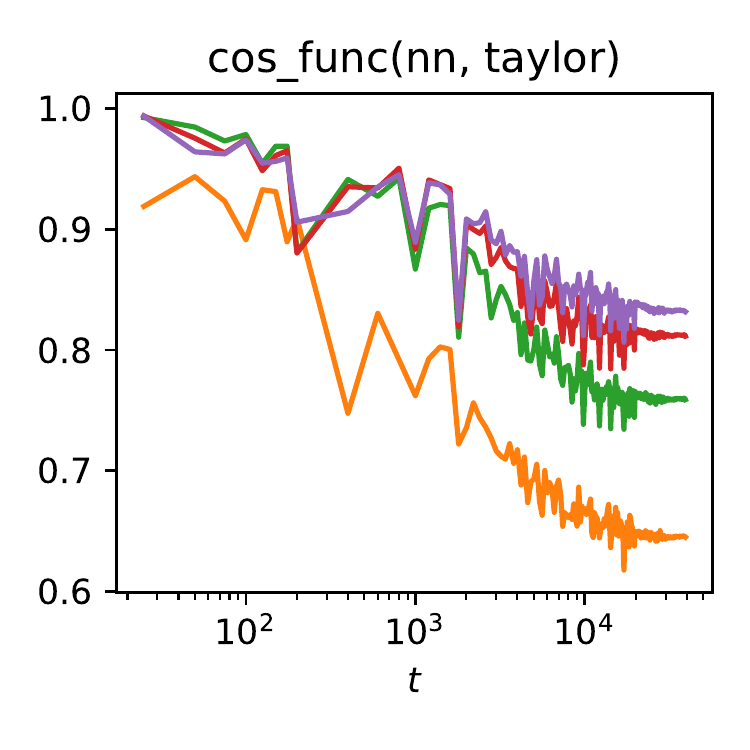}
    \caption{$k$-th order Taylorized training approximates full
      training increasingly better with $k$ on the WRNTHIN
      model. Training statistics are plotted for \{\textbf{{\color{C0}
          full}, {\color{C1} linearized}, {\color{C2} quadratic},
        {\color{C3} cubic}, {\color{C4} quartic}}\} models. Left to
      right: (1) training loss; (2) test accuracy; (3) cosine
      similarity between Taylorized and full training in the parameter
      space; (4) cosine similarity between Taylorized and full
      training in the function (logit) space. All models are trained on
      CIFAR-10 for 39200 steps, and a 10x learning rate decay happened
      at step \{19600, 29400\}.}
  \label{figure:wrnsmall-approximation}
\end{figure}

\begin{figure}[h]
  \centering
  \includegraphics[width=0.24\textwidth]{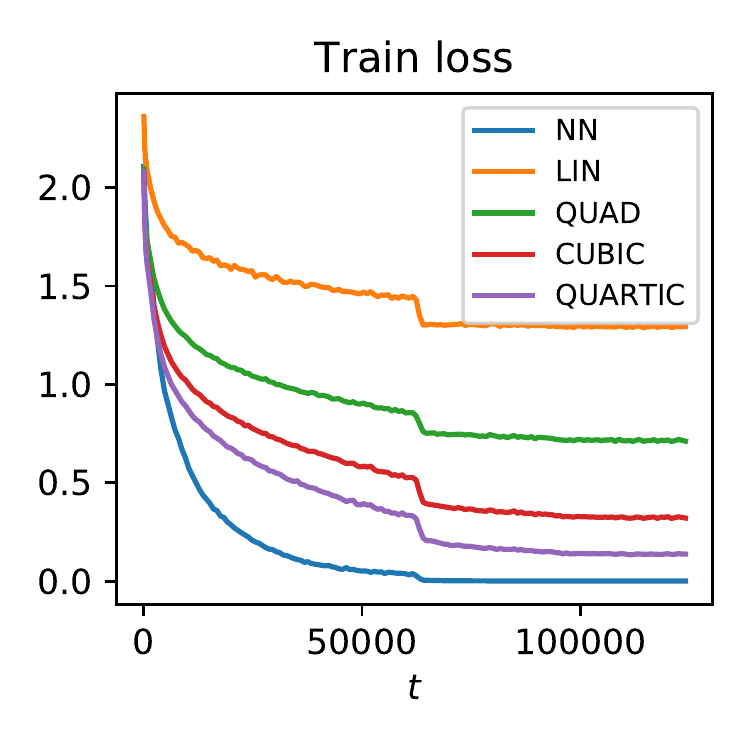}
  \includegraphics[width=0.24\textwidth]{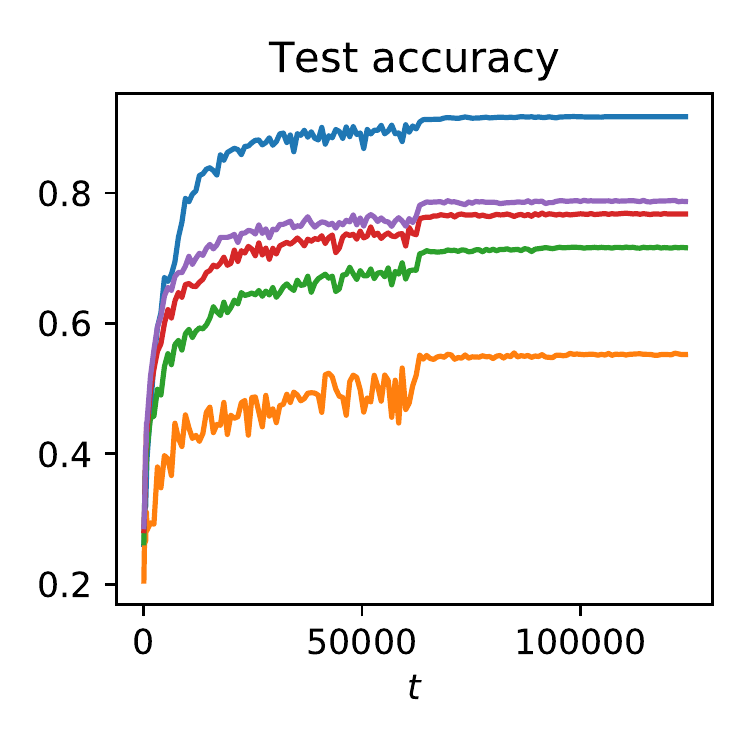}
  \includegraphics[width=0.24\textwidth]{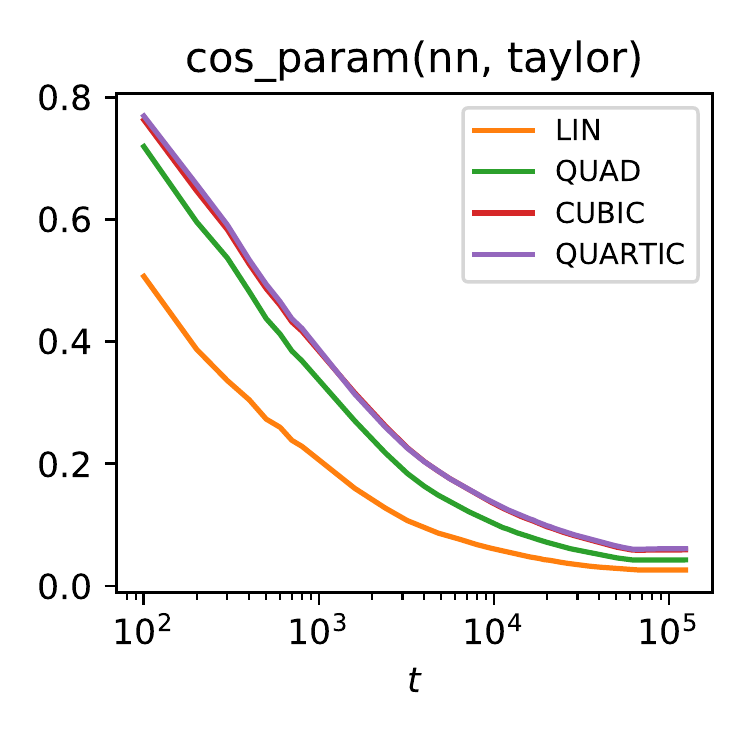}
  \includegraphics[width=0.24\textwidth]{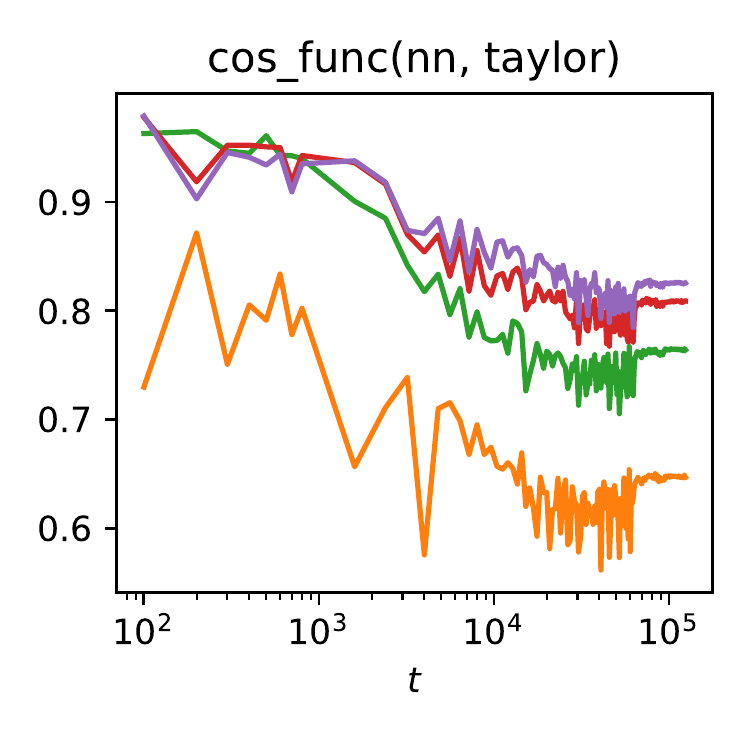}
  \caption{$k$-th order Taylorized training approximates full
    training increasingly better with $k$ on the WRNTHICK
    model. Training statistics are plotted for \{\textbf{{\color{C0}
        full}, {\color{C1} linearized}, {\color{C2} quadratic},
      {\color{C3} cubic}, {\color{C4} quartic}}\} models. Left to
    right: (1) training loss; (2) test accuracy; (3) cosine
    similarity between Taylorized and full training in the parameter
    space; (4) cosine similarity between Taylorized and full
    training in the function (logit) space. All models are trained on
    CIFAR-10 for 124000 steps, and a 10x learning rate decay happened
    at step \{62000, 93000\}.}
  \label{figure:wrnwide-approximation}
\end{figure}

\subsection{Effect of architectures / hyperparameters on Taylorized
  training}
\label{appendix:ablation}
In this section, we study the effect of various
architectural / hyperparameter choices on the approximation power of
Taylorized training. We observe the following for $k$-th order
Taylorized training for all $k\in\set{1,2,3,4}$:
\begin{itemize}
\item The approximation power gets slightly
  better (most significant in terms of the function-space
  similarity) when we increase the width (Section~\ref{appendix:width}).
\item The approximation power gets significantly better when we use a
  smaller learning rate switch from the standard parameterization to
  the NTK parameterization. However, under both settings, the learning
  slows down dramatically and do not give a reasonably-performing
  model in 100 epochs (Section~\ref{appendix:lr}).
\end{itemize}
We emphasize that (1) our observations largely agree with existing observations about linearized training~\citep{lee2019wide}, and (2) the point of these experiments is \emph{not} to demonstrate the superiority of training settings such as small learning rate / NTK parameterization, but rather to show the advantage of considering Taylorized training with a higher $k$ (instead of linearized training) in practical training regimes that do not typically use a small learning rate or NTK parameterization.

\subsubsection{Effect of width}
\label{appendix:width}
We first examine the effect of width on the approximation power of Taylorized training. For this purpose, we compare Taylorized training on the \{CNNTHIN, CNNTHICK\} models; the CNNTHICK model has the same base architecture but is 4x wider then the CNNTHIN model. We train all models under identical optimization setups (note this different from our setting in the main paper). Results are reported in Table~\ref{table:effect-width}.

Observe that widening the network has no significant effect on the test performance gap and the parameter-space cosine similarity, but increases the function-space cosine similarity for all Taylorized models.

\begin{table}[h]
  \small
  \centering
  \begin{tabular}{ll|l|l|l|l|l|l}
    \hline
    &         & \multicolumn{2}{c|}{test accuracy}   & \multicolumn{2}{c|}{cos\_param(nn, taylor)}      & \multicolumn{2}{c}{cos\_func(nn, taylor)}       \\
    \cline{3-8} 
    &         & 10 epochs & 80 epochs & 10 epochs & 80 epochs & 10 epochs & 80 epochs \\
    \hline
    \multicolumn{1}{c|}{\multirow{5}{*}{ \makecell{CNNTHIN \\ (4 layers, 128 channels)}}}  & LIN ($k=1$)  &   36.67\%    &  45.49\%   &  0.079    & 0.031  & 0.626  &  0.618              \\
    \multicolumn{1}{c|}{}                          & QUAD ($k=2$)    &  53.33\%    &  64.55\%    &  0.125  & 0.058  & 0.843  &  0.782              \\
    \multicolumn{1}{c|}{}                          & CUBIC ($k=3$)   &   59.53\%    &  70.50\%    &   0.150  & 0.078  &  0.894 & 0.823               \\
    \multicolumn{1}{c|}{}                          & QUARTIC ($k=4$) &  61.96\%     &  71.16\%    &  0.150  & 0.079  &  0.904  & 0.828               \\
    \multicolumn{1}{c|}{}                          & FULL NN   &  63.19\%  &  78.98\%           &  1.          & 1.       &  1.        &  1.              \\
    \hline
    \multicolumn{1}{l|}{\multirow{5}{*}{\makecell{CNNTHICK \\ (4 layers, 512 channels)}}} & LIN ($k=1$)     & 38.88\%     &  48.59\%   &  0.081     & 0.037        &   0.671  & 0.648               \\
    \multicolumn{1}{l|}{}                          & QUAD ($k=2$)    & 55.33\%       & 67.60\%    &  0.130   & 0.064    &  0.861    & 0.791                \\
    \multicolumn{1}{l|}{}                          & CUBIC ($k=3$)   &  61.72\%       & 71.86\%    &   0.143   & 0.075    &  0.911    & 0.834             \\
    \multicolumn{1}{l|}{}                          & QUARTIC ($k=4$) & 63.67\%     &  73.99\%   &  0.150    & 0.076    & 0.921      & 0.841               \\
    \multicolumn{1}{l|}{}                          & FULL NN   &  65.00\%          & 85.09\%  &  1.           & 1.          & 1.            & 1.              \\
    \hline
  \end{tabular}
  \caption{{\bf Effect of width on Taylorized training.} The two architectures are trained with the identical setting of batchsize=64, constant learning rate=$0.1$, and for 80 epochs. The cosine similarities are defined in Section~\ref{section:approximation-exp}. With a widened network, the function-space cosine similarity between neural net and Taylorized training becomes higher.} 
  \label{table:effect-width}
\end{table}

\subsubsection{Effect of learning rate and NTK parameterization}
\label{appendix:lr}
In our second set of experiments, we study the effect of learning rate and/or network parameterization. We choose a base setting of the CNNTHIN architecture with standard parameterization and constant learning rate 0.1, and tweak it by (1) lowering the learning rate to 0.01, or (2) switching to the NTK parameterization. Results are reported in Table~\ref{table:effect-lr}.

Observe that either lowering the learning rate or switching to NTK parameterization can dramatically improve the approximation power of Taylorized training (in terms of both function-space and parameter-space similarity), as well as  the performance gap. However, either setting significantly slows down training (as indicated by the test performance of the full neural net). 

\begin{table}[h]
  \small
  \centering
  \begin{tabular}{ll|l|l|l|l|l|l}
    \hline
    &         & \multicolumn{2}{c|}{test accuracy}   &
                                                       \multicolumn{2}{c|}{cos\_param(nn, taylor}      & \multicolumn{2}{c}{cos\_func(nn, taylor)}       \\
    \cline{3-8} 
    &         & 10 epochs & 100 epochs & 10 epochs & 100 epochs & 10 epochs & 100 epochs \\
    \hline
    \multicolumn{1}{c|}{\multirow{5}{*}{ \makecell[c]{CNNTHIN \\ standard param, lr$=0.1$}}}  & LIN ($k=1$)  &   32.94\%    &  41.22\%   &  0.156    & 0.049  & 0.649  &  0.605              \\
    \multicolumn{1}{c|}{}                          & QUAD ($k=2$)    &  44.47\%    &  59.90\%    &  0.241 & 0.083  & 0.907  &  0.803              \\
    \multicolumn{1}{c|}{}                          & CUBIC ($k=3$)   &   49.47\%    &  67.07\%    &   0.271  & 0.103  &  0.922 & 0.856               \\
    \multicolumn{1}{c|}{}                          & QUARTIC ($k=4$) &  52.13\%     &  68.64\%    &  0.272  & 0.115  &  0.920  & 0.867               \\
    \multicolumn{1}{c|}{}                          & FULL NN   &  48.85\%  &  77.47\%           &  1.          & 1.       &  1.        &  1.              \\
    \hline
    \multicolumn{1}{c|}{\multirow{5}{*}{\makecell[c]{CNNTHIN \\ standard param, lr$=0.01$}}} & LIN ($k=1$)     & 26.40\%     &  33.26\%   &  0.356    & 0.138  &   0.544  & 0.604               \\
    \multicolumn{1}{l|}{}                          & QUAD ($k=2$)    & 32.64\%       & 49.01\%    &  0.560   & 0.217    &  0.846    & 0.853                \\
    \multicolumn{1}{l|}{}                          & CUBIC ($k=3$)   &  33.44\%       & 54.34\%    &   0.606   & 0.257    &  0.956    & 0.914             \\
    \multicolumn{1}{l|}{}                          & QUARTIC ($k=4$) & 35.22\%     &  57.41\%   &  0.622   & 0.272    & 0.971      & 0.926               \\
    \multicolumn{1}{l|}{}                          & FULL NN   &  35.66\%          & 58.84\%  &  1.           & 1.          & 1.            & 1.              \\
    \hline
    \multicolumn{1}{c|}{\multirow{5}{*}{\makecell[c]{CNNTHIN \\ NTK param, lr=$0.1$}}} & LIN ($k=1$)     & 22.77\%     &  27.43\%   &  0.832    & 0.495        &   0.999  & 0.919               \\
    \multicolumn{1}{l|}{}                          & QUAD ($k=2$)    & 22.56\%       & 31.77\%    &  0.898   & 0.692    &  0.999    & 0.987                \\
    \multicolumn{1}{l|}{}                          & CUBIC ($k=3$)   &  22.70\%       & 33.23\%    &   0.899   & 0.707    &  0.999    & 0.993             \\
    \multicolumn{1}{l|}{}                          & QUARTIC ($k=4$) & 22.71\%     &  33.53\%   &  0.899    & 0.708    & 0.999      & 0.992               \\
    \multicolumn{1}{l|}{}                          & FULL NN   &  21.12\%          & 32.58\%  &  1.           & 1.          & 1.            & 1.              \\
    \hline
  \end{tabular}
  \caption{{\bf Effect of learning rate and parameterization on Taylorized training.} All models are trained with the identical settings of batchsize=256 and for 100 epochs. The cosine similarities are defined in Section~\ref{section:approximation-exp}. Observe that using a lower learning rate or the NTK parameterization can significantly imporve the approximation accuracy of Taylorized models, but meanwhile will slow down the vanilla neural network training.}
  \label{table:effect-lr}
\end{table}

% \section{Do \emph{not} have an appendix here}

% \textbf{\emph{Do not put content after the references.}}
% %
% Put anything that you might normally include after the references in a separate
% supplementary file.

% We recommend that you build supplementary material in a separate document.
% If you must create one PDF and cut it up, please be careful to use a tool that
% doesn't alter the margins, and that doesn't aggressively rewrite the PDF file.
% pdftk usually works fine. 

% \textbf{Please do not use Apple's preview to cut off supplementary material.} In
% previous years it has altered margins, and created headaches at the camera-ready
% stage. 
%%%%%%%%%%%%%%%%%%%%%%%%%%%%%%%%%%%%%%%%%%%%%%%%%%%%%%%%%%%%%%%%%%%%%%%%%%%%%%%
%%%%%%%%%%%%%%%%%%%%%%%%%%%%%%%%%%%%%%%%%%%%%%%%%%%%%%%%%%%%%%%%%%%%%%%%%%%%%%%

\end{document}